\documentclass{article}
\usepackage{fullpage}

\usepackage{graphicx} 
\usepackage{subfigure}

\usepackage{natbib}

\usepackage{algorithm}
\usepackage{algorithmic}

\usepackage{amsthm,amsfonts,amsmath,amssymb,epsfig,color,float,verbatim}

\newtheorem{theorem}{Theorem}
\newtheorem{proposition}{Proposition}
\newtheorem{lemma}{Lemma}

\newcommand{\reals}{\mathbb{R}}
\newcommand{\E}{\mathbb{E}}

\newcommand{\var}{\text{Var}}

\newcommand{\Prob}[1]{{\mathrm{Prob} \left(#1\right)}}

\newcommand{\bx}{\mathbf{x}}
\newcommand{\bw}{\mathbf{w}}

\newcommand{\bv}{\mathbf{v}}
\newcommand{\bz}{\mathbf{z}}

\newcommand{\bg}{\mathbf{g}}

\newcommand{\Ocal}{\mathcal{O}}

\newcommand{\Wcal}{\mathcal{W}}

\newcommand{\norm}[1]{\|#1\|}
\newcommand{\inner}[1]{\langle#1\rangle}

\newcommand{\secref}[1]{Sec.~\ref{#1}}

\newcommand{\figref}[1]{Fig.~\ref{#1}}
\renewcommand{\eqref}[1]{Eq.~(\ref{#1})}
\newcommand{\lemref}[1]{Lemma~\ref{#1}}

\newcommand{\thmref}[1]{Thm.~\ref{#1}}
\newcommand{\propref}[1]{Proposition~\ref{#1}}
\newcommand{\appref}[1]{Appendix~\ref{#1}}

\begin{document}

\title{Making Gradient Descent Optimal\\ for Strongly Convex Stochastic Optimization}
\author{Alexander Rakhlin\\rakhlin@wharton.upenn.edu\\University of Pennsylvania
\and
Ohad Shamir\\ohadsh@microsoft.com\\Microsoft Research New England
\and
Karthik Sridharan\\skarthik@wharton.upenn.edu\\University of Pennsylvania}

\date{}
\maketitle

\begin{abstract}
Stochastic gradient descent (SGD) is a simple and popular method to solve stochastic optimization problems which arise in machine learning. For strongly convex problems, its convergence rate was known to be $\Ocal(\log(T)/T)$, by running SGD for $T$ iterations and returning the average point. However, recent results showed that using a different algorithm, one can get an optimal $\Ocal(1/T)$ rate. This might lead one to believe that standard SGD is suboptimal, and maybe should even be replaced as a method of choice. In this paper, we investigate the optimality of SGD in a stochastic setting. We show that for smooth problems, the algorithm attains the optimal $\Ocal(1/T)$ rate. However, for non-smooth problems, the convergence rate with averaging might really be $\Omega(\log(T)/T)$, and this is not just an artifact of the analysis. On the flip side, we show that a simple modification of the averaging step suffices to recover the $\Ocal(1/T)$ rate, and no other change of the algorithm is necessary. We also present experimental results which support our findings, and point out open problems.
\end{abstract}

\section{Introduction}

Stochastic gradient descent (SGD) is one of the simplest and most popular first-order methods to solve convex learning problems. Given a convex loss function and a training set of $T$ examples, SGD can be used to obtain a sequence of $T$ predictors, whose average has a generalization error which converges (with $T$) to the optimal one in the class of predictors we consider. The common framework to analyze such first-order algorithms is via stochastic optimization, where our goal is to optimize an unknown convex function $F$, given only unbiased estimates of $F$'s subgradients (see \secref{sec:prelim} for a more precise definition).

An important special case is when $F$ is strongly convex (intuitively, can be lower bounded by a quadratic function). Such functions arise, for instance, in Support Vector Machines and other regularized learning algorithms. For such problems, there is a well-known $\Ocal(\log(T)/T)$ convergence guarantee for SGD with averaging. This rate is obtained using the analysis of the algorithm in the harder setting of online learning \cite{HazAgKal07}, combined with an online-to-batch conversion (see \cite{HazanKa11} for more details).

Surprisingly, a recent paper by Hazan and Kale \cite{HazanKa11} showed that in fact, an $\Ocal(\log(T)/T)$ is not the best that one can achieve for strongly convex stochastic problems. In particular, an optimal $\Ocal(1/T)$ rate can be obtained using a different algorithm, which is somewhat similar to SGD but is more complex (although with comparable computational complexity)\footnote{Roughly speaking, the algorithm divides the $T$ iterations into exponentially increasing epochs, and runs stochastic gradient descent with averaging on each one. The resulting point of each epoch is used as the starting point of the next epoch. The algorithm returns the resulting point of the last epoch.}. A very similar algorithm was also presented recently by Juditsky and Nesterov \cite{JudNes11}.

These results left an important gap: Namely, whether the true convergence rate of SGD, possibly with some sort of averaging, might also be $\Ocal(1/T)$, and the known $\Ocal(\log(T)/T)$ result is just an artifact of the analysis. Indeed, the whole motivation of \cite{HazanKa11} was that the standard online analysis is too loose to analyze the stochastic setting properly. Perhaps a similar looseness applies to the analysis of SGD as well? This question has immediate practical relevance: if the new algorithms enjoy a better rate than SGD, it might indicate they will work better in practice, and that practitioners should abandon SGD in favor of them.

In this paper, we study the convergence rate of SGD for stochastic strongly convex problems, with the following contributions:
\begin{itemize}
    \item First, we extend known results to show that if $F$ is not only strongly convex, but also smooth (with respect to the optimum), then SGD with and without averaging achieves the optimal $\Ocal(1/T)$ convergence rate.
    \item We then show that for non-smooth $F$, there are cases where the convergence rate of SGD with averaging is $\Omega(\log(T)/T)$. In other words, the $\Ocal(\log(T)/T)$ bound for general strongly convex problems is real, and not just an artifact of the currently-known analysis.
    \item However, we show that one can recover the optimal $\Ocal(1/T)$ convergence rate by a simple modification of the averaging step: Instead of averaging of $T$ points, we only average the last $\alpha T$ points, where $\alpha\in (0,1)$ is arbitrary. Thus, to obtain an optimal rate, one does not need to use an algorithm significantly different than SGD, such as those discussed earlier.
    \item We perform an empirical study on both artificial and real-world data, which supports our findings.
\end{itemize}
Following the paradigm of \cite{HazanKa11}, we analyze the algorithm directly in the stochastic setting, and avoid an online analysis with an online-to-batch conversion. Our rate upper bounds are shown to hold in expectation, but we also sketch how we can obtain high-probability bounds (up to a $\log(\log(T))$ factor). While the focus here is on getting the optimal rate in terms of $T$, we note that our upper bounds are also optimal in terms of other standard problem parameters, such as the strong convexity parameter and the variance of the stochastic gradients.

In terms of related work, we note that the performance of SGD in a stochastic setting has been extensively researched in stochastic approximation theory (see for instance \cite{KushnerYin03}). However, these results are usually obtained under smoothness assumptions, and are often asymptotic, so we do not get an explicit bound in terms of $T$ which applies to our setting. We also note that a finite-sample analysis of SGD in the stochastic setting was recently presented in \cite{BachMoulines11}. However, the focus there was different than ours, and also obtained bounds which hold only in expectation rather than in high probability. More importantly, the analysis was carried out under stronger smoothness assumptions than our analysis, and to the best of our understanding, does not apply to general, possibly non-smooth, strongly convex stochastic optimization problems. For example, smoothness assumptions may not cover the application of SGD to support vector machines (as in \cite{ShaSiSreCo11}), since it uses a non-smooth loss function, and thus the underlying function $F$ we are trying to stochastically optimize may not be smooth.

\section{Preliminaries}\label{sec:prelim}

We use bold-face letters to denote vectors. Given some vector $\bw$, we use $w_i$ to denote its $i$-th coordinate. Similarly, given some indexed vector $\bw_t$, we let $w_{t,i}$ denote its $i$-th coordinate. We let $\mathbf{1}_{A}$ denote the indicator function for some event $A$.

We consider the standard setting of convex stochastic optimization, using first-order methods. Our goal is to minimize a convex function $F$ over some convex domain $\Wcal$ (which is assumed to be a subset of some Hilbert space). However, we do not know $F$, and the only information available is through a stochastic gradient oracle, which given some $\bw\in\Wcal$, produces a vector $\hat{\bg}$, whose expectation $\E[\hat{\bg}]=\bg$ is a subgradient of $F$ at $\bw$. Using a bounded number $T$ of calls to this oracle, we wish to find a point $\bw_T$ such that $F(\bw_t)$ is as small as possible. In particular, we will assume that $F$ attains a minimum at some $\bw^{*}\in\Wcal$, and our analysis provides bounds on $F(\bw_t)-F(\bw^{*})$ either in expectation or in high probability. The application of this framework to learning is straightforward (see for instance \cite{ShalShamSrebSri09b}): given a hypothesis class $\Wcal$ and a set of $T$ i.i.d. examples, we wish to find a predictor $\bw$ whose expected loss $F(\bw)$ is close to optimal over $\Wcal$. Since the examples are chosen i.i.d., the subgradient of the loss function with respect to any individual example can be shown to be an unbiased estimate of a subgradient of $F$.

We will focus on an important special case of the problem, characterized by $F$ being a \emph{strongly convex} function. Formally, we say that a function $F$ is \emph{$\lambda$-strongly convex}, if for all $\bw,\bw'\in \Wcal$ and any subgradient $\bg$ of $F$ at $\bw$,
\begin{equation}\label{eq:strconvex}
F(\bw')\geq F(\bw)+\inner{\bg,\bw'-\bw}+\frac{\lambda}{2}\norm{\bw'-\bw}^2.
\end{equation}

Another possible property of $F$ we will consider is smoothness, at least with respect to the optimum $\bw^*$. Formally, a function $F$ is \emph{$\mu$-smooth with respect to $\bw^*$} if for all $\bw\in\Wcal$,
\begin{equation}\label{eq:smoothness}
F(\bw) - F(\bw^*)\leq \frac{\mu}{2}\norm{\bw-\bw^*}^2.
\end{equation}
Such functions arise, for instance, in logistic and least-squares regression, and in general for learning linear predictors where the loss function has a Lipschitz-continuous gradient.

The algorithm we focus on is stochastic gradient descent (SGD). The SGD algorithm is parameterized by step sizes $\eta_1,\ldots,\eta_T$, and is defined as follows (below, we assume for simplicity that the algorithm is initialized at $\mathbf{0}\in \Wcal$, following common practice).
\begin{enumerate}
    \item Initialize $\bw_1=\mathbf{0}$
    \item For $t=1,\ldots,T$:
    \begin{itemize}
        \item Query the stochastic gradient oracle at $\bw_t$ to get a random $\hat{\bg}_t$ such that $\E[\hat{\bg}_t]=\bg_t$ is a subgradient of $F$ at $\bw_t$.
        \item Let $\bw_{t+1}=\Pi_{\Wcal}(\bw_t-\eta_t \hat{\bg}_t)$, where $\Pi_{\Wcal}$ is the projection operator on $\Wcal$.
    \end{itemize}
\end{enumerate}
This algorithm returns a sequence of points $\bw_1,\ldots,\bw_T$. To obtain a single point, one can use several strategies. Perhaps the simplest one is to return the last point, $\bw_{T+1}$. Another procedure, for which the standard online analysis of SGD applies \cite{HazAgKal07}, is to return the average point
\[
\bar{\bw}_T = \frac{1}{T}(\bw_1+\ldots+\bw_T).
\]

In terms of the step size, we note that the appropriate regime to consider is $\eta_t=\Theta(1/t)$ (see \appref{app:etat} for a fuller discussion of this). In particular, we will assume for the sake of our upper bounds that $\eta_t = 1/\lambda t$. This assumption simplifies the analysis substantially, while not losing much in terms of generality. To see why, suppose the step sizes are actually $c/\lambda t$ for some\footnote{If the step size is too small and $c$ is much smaller than $1$, the SGD analysis is known to fail (\cite{JuLaNeSha09}).} $c\geq 1$, and let $\lambda' = \lambda/c$. Then this step size is equivalent to $1/(\lambda't)$. Since any $\lambda$-strongly convex function is also $\lambda'$-strongly convex (as $\lambda \geq \lambda'$), then we can just analyze the algorithm's behavior as if we run it on a $\lambda'$-strongly convex function, using the default step size $1/\lambda' t$. If so desired, one can then substitute $\lambda/c$ instead of $\lambda'$ in the final bound, to see the upper bound in terms of $\lambda$ and $c$.

In general, we will assume that regardless of how the iterates evolve, it holds that $\E[\norm{\hat{\bg}_t}^2]\leq G^2$ for some fixed constant $G$. Note that this is a somewhat weaker assumption than \cite{HazanKa11}, which required that $\norm{\hat{\bg}_t}^2\leq G^2$ with probability $1$, since we focus here mostly on bounds which hold in expectation. These types of assumptions are common in the literature, and are generally implied by taking $\Wcal$ to be a bounded domain, or alternatively, assuming that $\bw_1$ is initialized not too far from $\bw^{*}$ and $F$ satisfies certain technical conditions (see for instance the proof of Theorem 1 in \cite{ShaSiSreCo11}).

Full proofs of our results are provided in \appref{app:proofs}.

\section{Smooth Functions}\label{sec:smooth}

We begin by considering the case where the expected function $F(\cdot)$ is both strongly convex and smooth with respect to $\bw^*$. Our starting point is to show a $\Ocal(1/T)$ for the \emph{last} point obtained by SGD. This result is well known in the literature (see for instance \cite{JuLaNeSha09}) and we include a proof for completeness. Later on, we will
show how to extend it to a high-probability bound.

\begin{theorem}\label{thm:last}
Suppose $F$ is $\lambda$-strongly convex and $\mu$-smooth with respect to $\bw^*$ over a convex set $\Wcal$, and that $\E[\norm{\hat{\bg}_t}^2]\leq G^2$. Then if we pick $\eta_t = 1/\lambda t$, it holds for any $T$ that
\[
\E[F(\bw_T)-F(\bw^{*})] \leq \frac{2\mu G^2}{\lambda^2 T}.
\]
\end{theorem}
The theorem is an immediate corollary of the following key lemma, and the definition of $\mu$-smoothness with respect to $\bw^*$.
\begin{lemma}\label{lem:close}
Suppose $F$ is $\lambda$-strongly convex over a convex set $\Wcal$, and that $\E[\norm{\hat{\bg}_t}^2]\leq G^2$. Then if we pick $\eta_t = 1/\lambda t$, it holds for any $T$ that
\[
\E\left[\norm{\bw_T-\bw^{*}}^2\right] \leq \frac{4 G^2}{\lambda^2 T}.
\]
\end{lemma}

We now turn to discuss the behavior of the average point $\bar{\bw}_T= (\bw_1+\ldots+\bw_T)/T$, and show that for smooth $F$, it also enjoys an optimal $\Ocal(1/T)$ convergence rate.

\begin{theorem}\label{thm:average}
Suppose $F$ is $\lambda$-strongly convex and $\mu$-smooth with respect to $\bw^*$ over a convex set $\Wcal$, and that $\E[\norm{\hat{\bg}_t}^2]\leq G^2$. Then if we pick $\eta_t = 1/\lambda t$,
\[
\E[F(\bar{\bw}_T)-F(\bw^{*})] \leq \frac{16 \mu G^2}{\lambda^2 T}.
\]
\end{theorem}
A rough proof intuition is the following: \lemref{lem:close} implies that the Euclidean distance of $\bw_t$ from $\bw^{*}$ is on the order of $1/\sqrt{t}$, so the squared distance of $\bar{\bw}_T$ from $\bw^{*}$ is on the order of $((1/T)\sum_{t=1}^{T}1/\sqrt{t})^2 \approx 1/T$, and the rest follows from smoothness.

\section{Non-Smooth Functions}\label{sec:nonsmooth}

We now turn to the discuss the more general case where the function $F$ may not be smooth (i.e. there is no constant $\mu$ which satisfies \eqref{eq:smoothness} uniformly for all $\bw\in\Wcal$). In the context of learning, this may happen when we try to learn a predictor with respect to a non-smooth loss function, such as the hinge loss.

As discussed earlier, SGD with averaging is known to have a rate of at most $\Ocal(\log(T)/T)$. In the previous section, we saw that for smooth $F$, the rate is actually $\Ocal(1/T)$. Moreover, \cite{HazanKa11} showed that for using a different algorithm than SGD, one can obtain a rate of $\Ocal(1/T)$ even in the non-smooth case. This might lead us to believe that an $\Ocal(1/T)$ rate for SGD is possible in the non-smooth case, and that the $\Ocal(\log(T)/T)$ analysis is simply not tight.

However, this intuition turns out to be wrong. Below, we show that there are strongly convex stochastic optimization problems in Euclidean space, in which the convergence rate of SGD with averaging is lower bounded by $\Omega(\log(T)/T)$. Thus, the logarithm in the bound is not merely a shortcoming in the standard online analysis of SGD, but is really a property of the algorithm.

We begin with the following relatively simple example, which shows the essence of the idea. Let $F$ be the $1$-strongly convex function
\[
F(\bw) = \frac{1}{2}\norm{\bw}^2+w_1,
\]
over the domain $\Wcal=[0,1]^d$, which has a global minimum at $\mathbf{0}$. Suppose the stochastic gradient oracle, given a point $\bw_t$, returns the gradient estimate
$
\hat{\bg}_t = \bw_t+(Z_t,0,\ldots,0),
$
where $Z_t$ is uniformly distributed over $[-1,3]$. It is easily verified that $\E[\hat{\bg}_t]$ is a subgradient of $F(\bw_t)$, and that $\E[\norm{\hat{\bg}_t}^2]\leq d+5$ which is a bounded quantity for fixed $d$.

The following theorem implies in this case, the convergence rate of SGD with averaging has a $\Omega(\log(T)/T)$ lower bound. The intuition for this is that the global optimum lies at a corner of $\Wcal$, so SGD ``approaches'' it only from one direction. As a result, averaging the points returned by SGD actually hurts us.

\begin{theorem}\label{thm:simple}
Consider the strongly convex stochastic optimization problem presented above. If SGD is initialized at any point in $\Wcal$, and ran with $\eta_t = c/t$, then for any $T\geq T_0+1$, where $T_0=\max\{2,c/2\}$, we have
\[
\E[F(\bar{\bw}_{T})-F(\bw^{*})] ~\geq~\frac{c}{16T}\sum_{t=T_0}^{T-1}\frac{1}{t}.
\]
When $c$ is considered a constant, this lower bound is $\Omega(\log(T)/T)$.
\end{theorem}
While the lower bound scales with $c$, we remind the reader that one must pick $\eta_t=c/t$ with constant $c$ for an optimal convergence rate in general (see discussion in \secref{sec:prelim}).

This example is relatively straightforward but not fully satisfying, since it crucially relies on the fact that $\bw^*$ is on the border of $\Wcal$. In strongly convex problems, $\bw^*$ usually lies in the interior of $\Wcal$, so perhaps the $\Omega(\log(T)/T)$ lower bound does not hold in such cases. Our main result, presented below, shows that this is not the case, and that even if $\bw^*$ is well inside the interior of $\Wcal$, an $\Omega(\log(T)/T)$ rate for SGD with averaging can be unavoidable. The intuition is that we construct a non-smooth $F$, which forces $\bw_t$ to approach the optimum from just one direction, creating the same effect as in the previous example.

In particular, let $F$ be the $1$-strongly convex function
\[
F(\bw) = \frac{1}{2}\norm{\bw}^2+\begin{cases} w_1 & w_1\geq 0 \\ -7w_1 & w_1<0\end{cases},
\]
over the domain $\Wcal=[-1,1]^d$, which has a global minimum at $\mathbf{0}$. Suppose the stochastic gradient oracle, given a point $\bw_t$, returns the gradient estimate
\[
\hat{\bg}_t = \bw_t+\begin{cases} (Z_t,0,\ldots,0) & w_1 \geq 0 \\ (-7,0,\ldots,0) & w_1<0\end{cases},
\]
where $Z_t$ is a random variable uniformly distributed over $[-1,3]$. It is easily verified that $\E[\hat{\bg}_t]$ is a subgradient of $F(\bw_t)$, and that $\E[\norm{\hat{\bg}_t}^2]\leq d+63$ which is a bounded quantity for fixed $d$.

\begin{theorem}\label{thm:involved}
Consider the strongly convex stochastic optimization problem presented above. If SGD is initialized at any point $\bw_1$ with $w_{1,1}\geq 0$, and ran with $\eta_t = c/t$, then for any $T\geq T_0+2$, where $T_0=\max\{2,6c+1\}$, we have
\[
\E\left[F(\bar{\bw}_{T})-F(\bw^{*})\right] ~\geq~\frac{3c}{16 T}\sum_{t=T_0+2}^{T}\left(\frac{1}{t}\right)-\frac{T_0}{T}.
\]
When $c$ is considered a constant, this lower bound is $\Omega(\log(T)/T)$.
\end{theorem}
We note that the requirement of $w_{1,1}\geq 0$ is just for convenience, and the analysis also carries through, with some second-order factors, if we let $w_{1,1}<0$.

\section{Recovering an $\Ocal(1/T)$ Rate for SGD with $\alpha$-Suffix Averaging}\label{sec:suffix}

In the previous section, we showed that SGD with averaging may have a rate of $\Omega(\log(T)/T)$ for non-smooth $F$. To get the optimal $\Ocal(1/T)$ rate for any $F$, we might turn to the algorithms of \cite{HazanKa11} and \cite{JudNes11}. However, these algorithms constitute  a significant departure from standard SGD. In this section, we show that it is actually possible to get an $\Ocal(1/T)$ rate using a much simpler modification of the algorithm: given the sequence of points $\bw_1,\ldots,\bw_T$ provided by SGD, instead of returning the average $\bar{\bw}_T = (\bw_1+\ldots+\bw_T)/T$, we average and return just a suffix, namely
\[
\bar{\bw}^{\alpha}_T = \frac{\bw_{(1-\alpha)T+1}+\ldots+\bw_T}{\alpha T},
\]
for some constant $\alpha\in (0,1)$ (assuming $\alpha T$ and $(1-\alpha)T$ are integers). We call this procedure \emph{$\alpha$-suffix averaging}.

\begin{theorem}\label{thm:alphaaveraging}
Consider SGD with $\alpha$-suffix averaging as described above, and with step sizes $\eta_t=1/\lambda t$. Suppose $F$ is $\lambda$-strongly convex, and that $\E[\norm{\hat{\bg}_t}^2]\leq G$ for all $t$. Then for any $T$, it holds that
\[
\E[F(\bar{\bw}^{\alpha}_{T})-F(\bw^*)] \leq \frac{2+2.5 \log\left(\frac{1}{1-\alpha}\right)}{\alpha}
\frac{G^2}{\lambda T}.
\]
\end{theorem}
Note that for any constant $\alpha\in (0,1)$, the bound above is $\Ocal(G^2/\lambda T)$. This matches the optimal guarantees in \cite{HazanKa11} up to constant factors. However, this is shown for standard SGD, as opposed to the more specialized algorithm of \cite{HazanKa11}. Also, it is interesting to note that this bound is comparable to the bound of \thmref{thm:last} for the last iterate, when $F$ is also smooth, as long as $\mu/\lambda = \Ocal(1)$. However, \thmref{thm:last} degrades as the function becomes less smooth. In contrast, \thmref{thm:alphaaveraging} implies that with an averaging scheme, we get an optimal rate even if the function is not smooth. Finally, we note that it might be tempting to use \thmref{thm:alphaaveraging} as a guide to choose the averaging window, by optimizing the bound for $\alpha$ (which turns out to be $\alpha\approx 0.65$). However, we note that the optimal value of $\alpha$ is dependent on the constants in the bound, which may not be the tightest or most ``correct'' ones.

\begin{proof}[Proof Sketch]
The proof combines the analysis of online gradient descent \cite{HazAgKal07} and \lemref{lem:close}. In particular, starting as in the proof of \lemref{lem:close}, and extracting the inner products, we get
\begin{align}
&\sum_{t=(1-\alpha) T+1}^{T}\E[\inner{\bg_t,\bw_t-\bw^{*}}]
~\leq~ \sum_{t=(1-\alpha) T+1}^{T}\frac{\eta_t G^2}{2} +\notag\\
&~~
\sum_{t=(1-\alpha) T+1}^{T}\left(\frac{\E[\norm{\bw_t-\bw^{*}}^2]}{2\eta_t}-\frac{\E[\norm{\bw_{t+1}-\bw^{*}}^2]}{2\eta_t}\right).
\label{eq:half1}
\end{align}
Rearranging the r.h.s., and using the convexity of $F$ to relate the l.h.s. to $\E[F(\bar{\bw}^{\alpha}_{T})-F(\bw^{*})]$, we get a convergence upper bound of
\begin{align*}
\frac{1}{2\alpha T}&\left(\frac{\E[\norm{\bw_{(1-\alpha)T+1}
-\bw^*}^2]}{\eta_{(1-\alpha)T+1}}+G^2\sum_{t=(1-\alpha)T+1}^{T}\eta_t
\right. \\
&+
\left.
\sum_{t=(1-\alpha)T+1}^{T}\E[\norm{\bw_t-\bw^*}^2]
\left(\frac{1}{\eta_t}-\frac{1}{\eta_{t-1}}\right)\right).
\end{align*}
\lemref{lem:close} tells us that with any strongly convex $F$, even non-smooth, we have $\E[\norm{\bw_t-\bw^{*}}^2]\leq \Ocal(1/t)$. Plugging this in and performing a few more manipulations, the result follows.
\end{proof}

One potential disadvantage of suffix averaging is that if we cannot store all the iterates $\bw_t$ in memory, then we need to know from which iterate $\alpha T$ to start computing the suffix average (in contrast, standard averaging can be computed ``on-the-fly'' without knowing the stopping time $T$ in advance). However, even if $T$ is not known, this can be addressed in several ways. For example, since our results are robust to the value of $\alpha$, it is really enough to guess when we passed some ``constant'' portion of all iterates. Alternatively, one can divide the rounds into exponentially increasing epochs, and maintain the average just of the current epoch. Such an average would always correspond to a constant-portion suffix of all iterates.

\section{High-Probability Bounds}\label{sec:high_probability}

All our previous bounds were on the \emph{expected} suboptimality $\E[F(\bw)-F(\bw^{*})]$ of an appropriate predictor $\bw$. We now outline how these results can be strengthened to bounds on $F(\bw_t)-F(\bw^{*})$ which hold with arbitrarily high probability $1-\delta$, with the bound depending logarithmically on $\delta$. Compared to our in-expectation bounds, they have an additional mild $\log(\log(T))$ factor (interestingly, a similar factor also appears in the analysis of \cite{HazanKa11}, and we do not know if it is necessary). The key result is the following strengthening of \lemref{lem:close}, under slightly stronger technical conditions.

\begin{proposition}\label{prop:closehighprob}
Let $\delta \in (0,1/e)$ and assume $T\geq 4$. Suppose $F$ is $\lambda$-strongly convex over a convex set $\Wcal$, and that $\norm{\hat{\bg}_t}^2\leq G^2$ with probability $1$. Then if we pick $\eta_t = 1/\lambda t$, it holds with probability at least $1-\delta$ that for any $t\leq T$,
\[
\norm{\bw_t-\bw^{*}}^2 \leq \frac{(624\log(\log(T)/\delta)+1)G^2}{\lambda^2 t}.
\]
\end{proposition}
To obtain high probability versions of \thmref{thm:last}, \thmref{thm:average}, and \thmref{thm:alphaaveraging}, one needs to use this lemma in lieu of \lemref{lem:close} in their proofs. This leads overall to rates of the form $\Ocal(\log(\log(T)/\delta)/T)$ which hold with probability $1-\delta$.

\section{Experiments}\label{sec:experiments}

We now turn to empirically study how the algorithms behave, and compare it to our theoretical findings.

We studied the following four algorithms:
\begin{enumerate}
    \item \textsc{Sgd-A}: Performing SGD and then returning the average point over all $T$ rounds.
    \item \textsc{Sgd-$\alpha$}: Performing SGD with $\alpha$-suffix averaging. We chose $\alpha=1/2$ - namely, we return the average point over the last $T/2$ rounds.
    \item \textsc{Sgd-L}: Performing SGD and returning the point obtained in the last round.
    \item \textsc{Epoch-Gd}: The optimal algorithm of \cite{HazanKa11} for strongly convex stochastic optimization.
\end{enumerate}

First, as a simple sanity check, we measured the performance of these algorithms on a simple, strongly convex stochastic optimization problem, which is also smooth. We define $\Wcal=[-1,1]^5$, and $F(\bw) = \norm{\bw}^2$. The stochastic gradient oracle, given a point $\bw$, returns the stochastic gradient $\bw+\bz$ where $\bz$ is uniformly distributed in $[-1,1]^5$. Clearly, this is an unbiased estimate of the gradient of $F$ at $\bw$. The initial point $\bw_1$ of all 4 algorithms was chosen uniformly at random from $\Wcal$. The results are presented in \figref{fig:artificialsmooth}, and it is clear that all 4 algorithms indeed achieve a $\Theta(1/T)$ rate, matching our theoretical analysis (\thmref{thm:last}, \thmref{thm:average} and \thmref{thm:alphaaveraging}). The results also seem to indicate that \textsc{Sgd-A} has a somewhat worse performance in terms of leading constants.

\begin{figure}
\begin{center}
\includegraphics[scale=0.4]{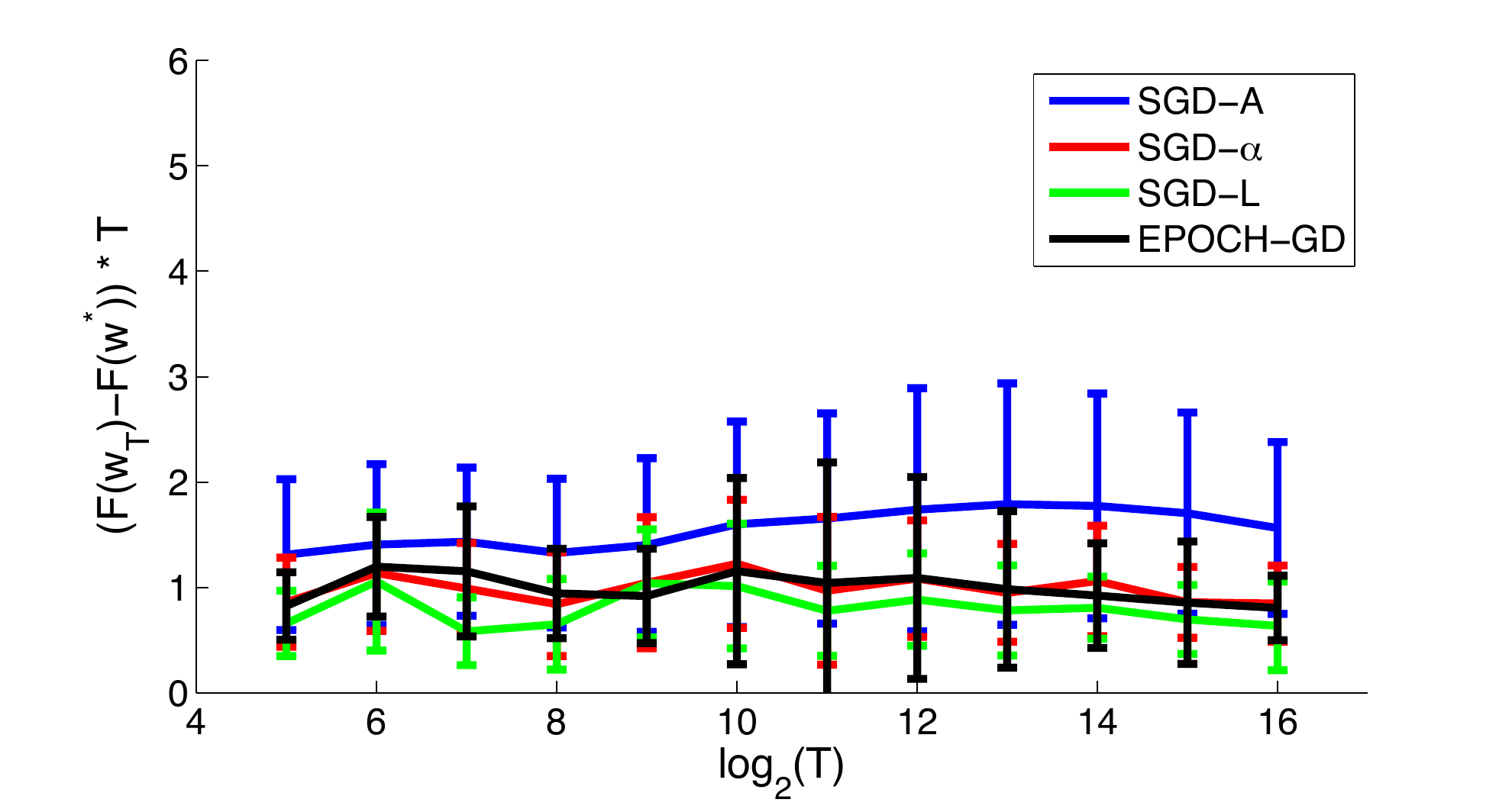}
\caption{Results for smooth strongly convex stochastic optimization problem. The experiment was repeated $10$ times, and we report the mean and standard deviation for each choice of $T$. The X-axis is the log-number of rounds $\log(T)$, and the Y-axis is $(F(\bw_T)-F(\bw^*))*T$. The scaling by $T$ means that a roughly constant graph corresponds to a $\Theta(1/T)$ rate, whereas a linearly increasing graph corresponds to a $\Theta(\log(T)/T)$ rate.}
\label{fig:artificialsmooth}
\end{center}
\end{figure}

Second, as another simple experiment, we measured the performance of the algorithms on the non-smooth, strongly convex problem described in the proof of \thmref{thm:involved}. In particular, we simulated this problem with $d=5$, and picked $\bw_1$ uniformly at random from $\Wcal$. The results are presented in \figref{fig:artificial}. As our theory indicates, \textsc{Sgd-A} seems to have an $\Theta(\log(T)/T)$ convergence rate, whereas the other 3 algorithms all seem to have the optimal $\Theta(1/T)$ convergence rate. Among these algorithms, the SGD variants \textsc{Sgd-L} and \textsc{Sgd-$\alpha$} seem to perform somewhat better than \textsc{Epoch-Gd}. Also, while the average performance of \textsc{Sgd-L} and \textsc{Sgd-$\alpha$} are similar, \textsc{Sgd-$\alpha$} has less variance. This is reasonable, considering the fact that \textsc{Sgd-$\alpha$} returns an average of many points, whereas \textsc{Sgd-L} return only the very last point.

\begin{figure}
\begin{center}
\includegraphics[scale=0.4]{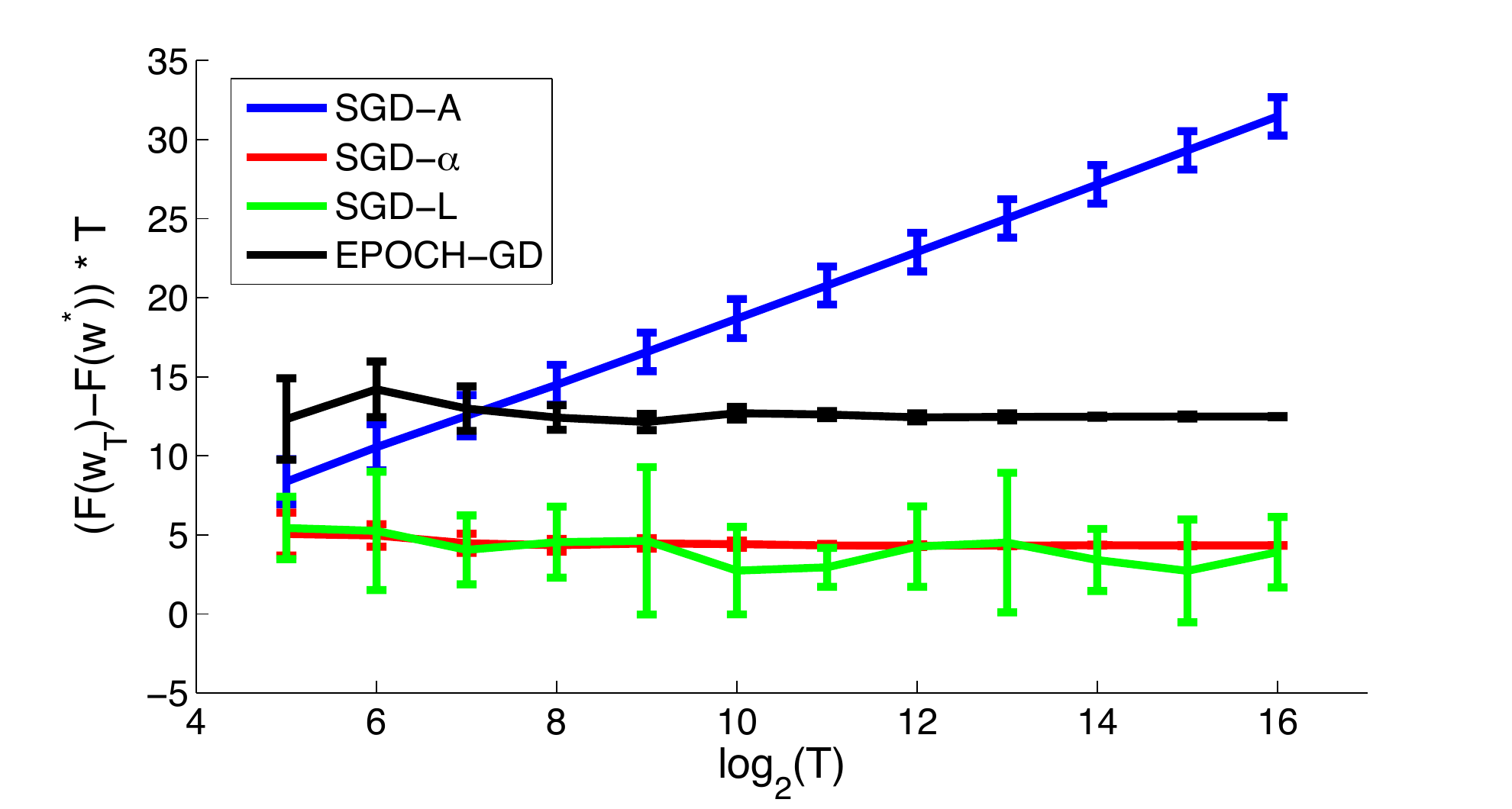}
\caption{Results for the non-smooth strongly convex stochastic optimization problem. The experiment was repeated $10$ times, and we report the mean and standard deviation for each choice of $T$. The X-axis is the log-number of rounds $\log(T)$, and the Y-axis is $(F(\bw_T)-F(\bw^*))*T$. The scaling by $T$ means that a roughly constant graph corresponds to a $\Theta(1/T)$ rate, whereas a linearly increasing graph corresponds to a $\Theta(\log(T)/T)$ rate.}
\label{fig:artificial}
\end{center}
\end{figure}

Finally, we performed a set of experiments on real-world data. We used the same 3 binary classification datasets (\textsc{ccat},\textsc{cov1} and \textsc{astro-ph}) used by \cite{ShaSiSreCo11} and \cite{Joachims06}, to test the performance of optimization algorithms for Support Vector Machines using linear kernels. Each of these datasets is composed of a training set and a test set. Given a training set of instance-label pairs, $\{\bx_i,y_i\}_{i=1}^{m}$, we defined $F$ to be the standard (non-smooth) objective function of Support Vector Machines, namely
\begin{equation}\label{eq:obj}
F(\bw) = \frac{\lambda}{2}\norm{\bw}^2 + \frac{1}{m}\sum_{i=1}^{m}\max\{0,1-y_i\inner{\bx_i,\bw}\}.
\end{equation}
Following \cite{ShaSiSreCo11} and \cite{Joachims06}, we took $\lambda=10^{-4}$ for \textsc{ccat}, $\lambda=10^{-6}$ for \textsc{cov1}, and $\lambda=5\times 10^{-5}$ for \textsc{astro-ph}. The stochastic gradient given $\bw_t$ was computed by taking a single randomly drawn training example $(\bx_i,y_i)$, and computing the gradient with respect to that example, namely
\[
\hat{\bg}_t = \lambda \bw_t - \mathbf{1}_{y_i\inner{\bx_i,\bw_t}\leq 1}y_i \bx_i.
\]
Each dataset comes with a separate test set, and we also report the objective function value with respect to that set (as in \eqref{eq:obj}, this time with $\{\bx_i,y_i\}$ representing the test set examples). All algorithms were initialized at $\bw_1=0$, with $\Wcal=\reals^d$ (i.e. no projections were performed - see the discussion in \secref{sec:prelim}).

The results of the experiments are presented in \figref{fig:astro},\figref{fig:ccat} and \figref{fig:cov}. In all experiments, \textsc{Sgd-A} performed the worst. The other 3 algorithms performed rather similarly, with \textsc{Sgd-$\alpha$} being slightly better on the \textsc{Cov1} dataset, and \textsc{Sgd-L} being slightly better on the other 2 datasets.

\begin{figure}
\begin{center}
\includegraphics[scale=0.4]{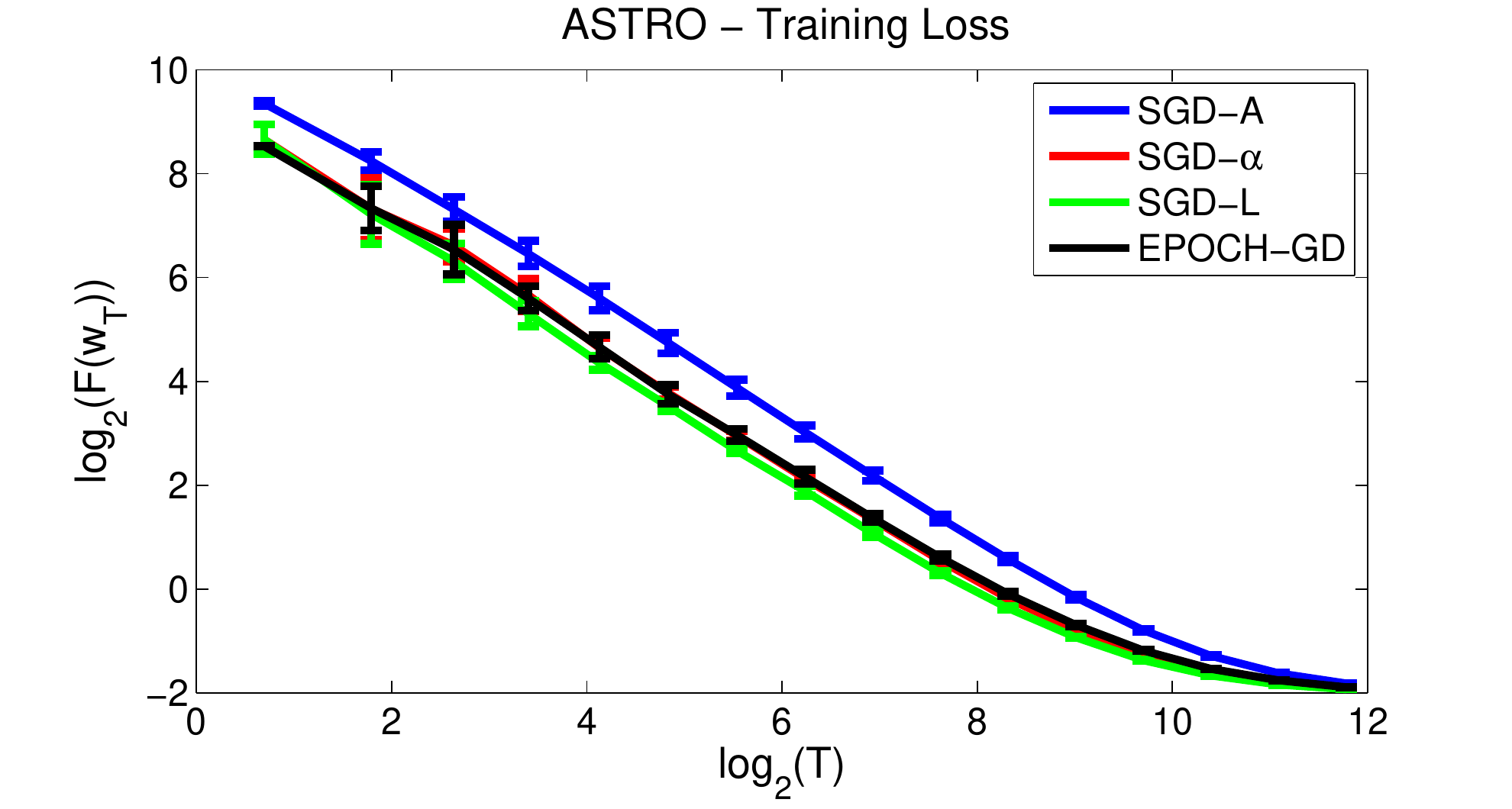}
\includegraphics[scale=0.4]{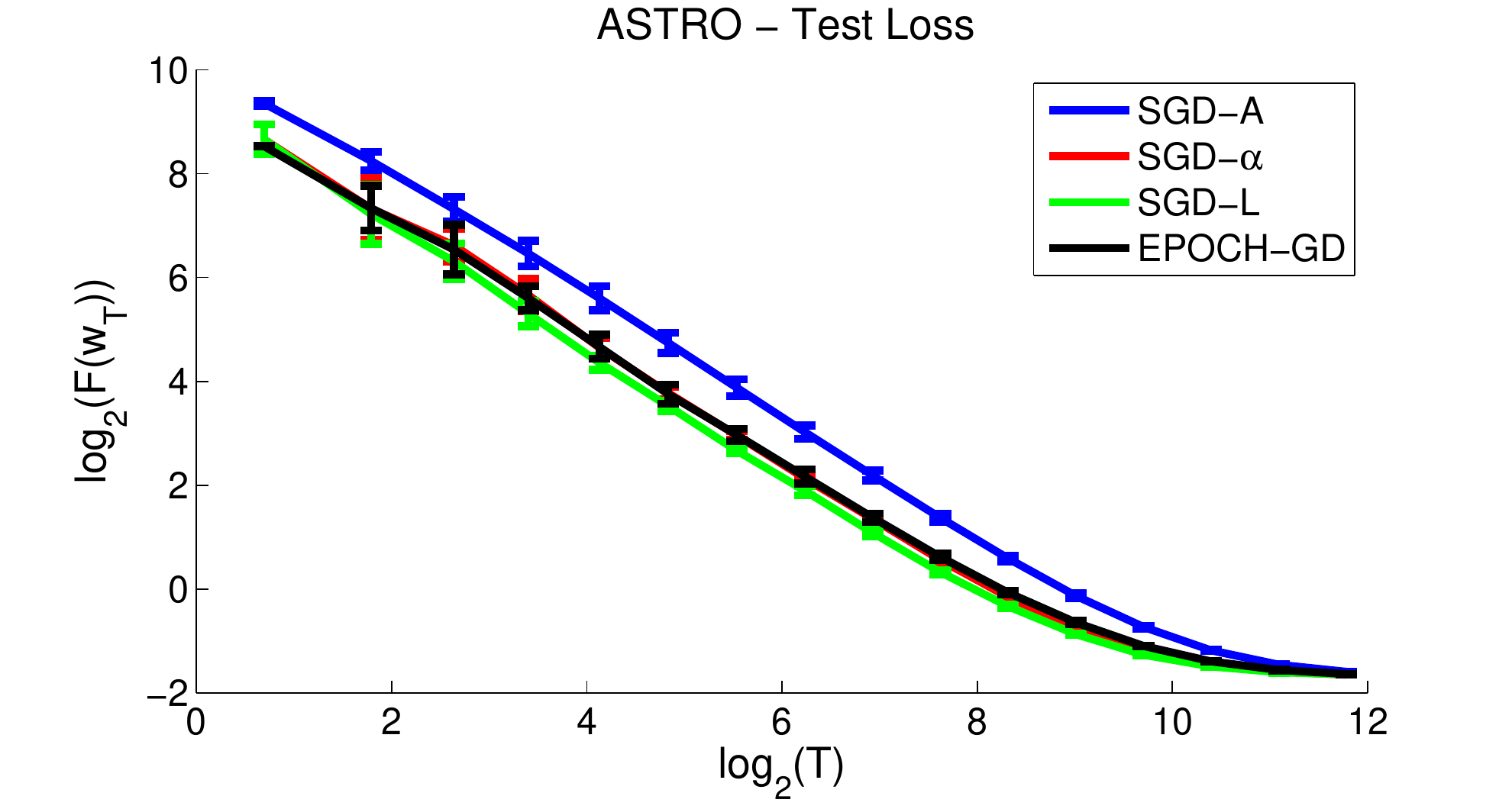}
\caption{Results for the \textsc{astro-ph} dataset. The left row refers to the average loss on the training data, and the right row refers to the average loss on the test data. Each experiment was repeated $10$ times, and we report the mean and standard deviation for each choice of $T$. The X-axis is the log-number of rounds $\log(T)$, and the Y-axis is the log of the objective function $\log(F(\bw_T))$.}
\label{fig:astro}
\end{center}
\end{figure}

\begin{figure}
\begin{center}
\includegraphics[scale=0.4]{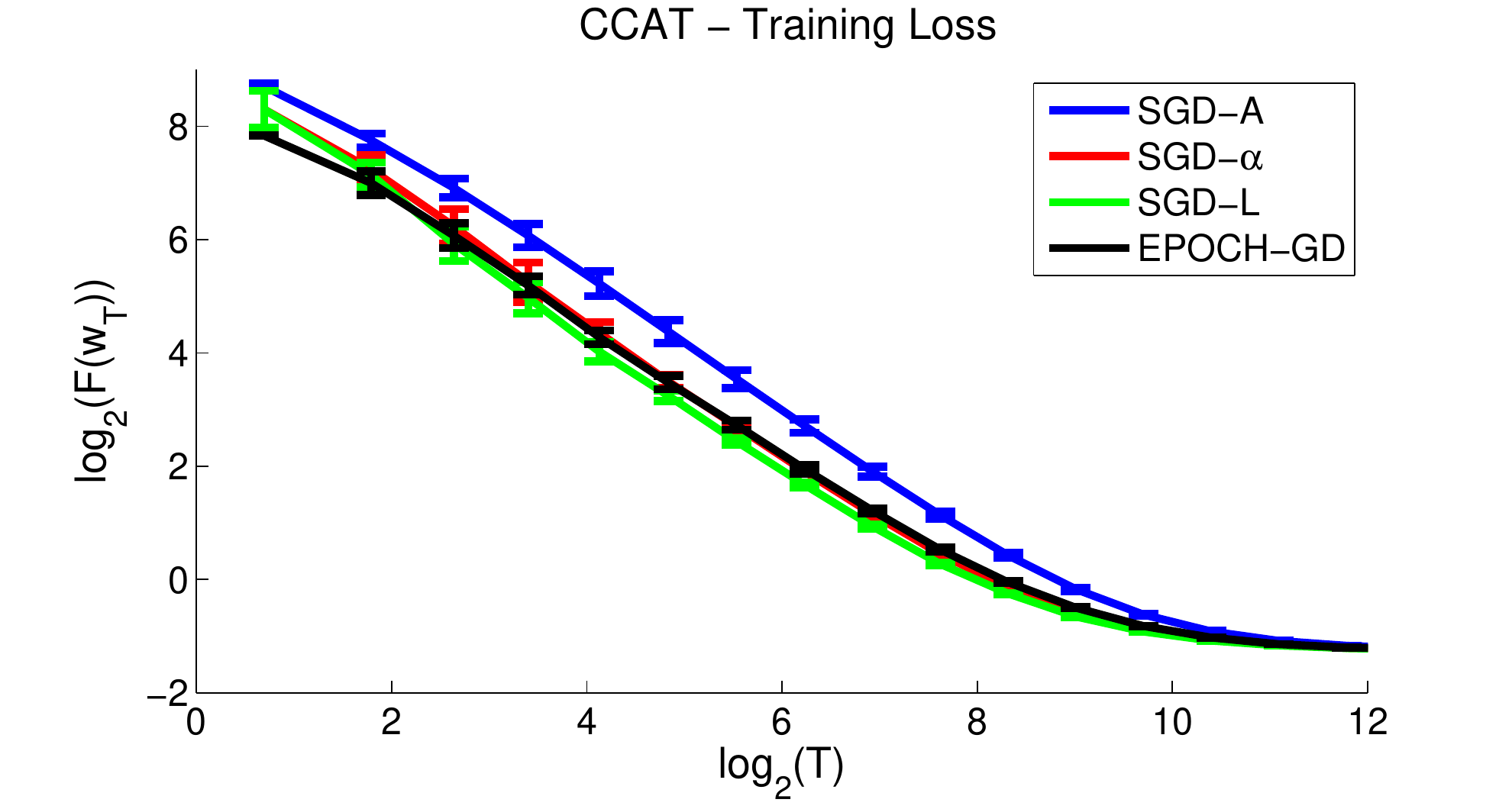}
\includegraphics[scale=0.4]{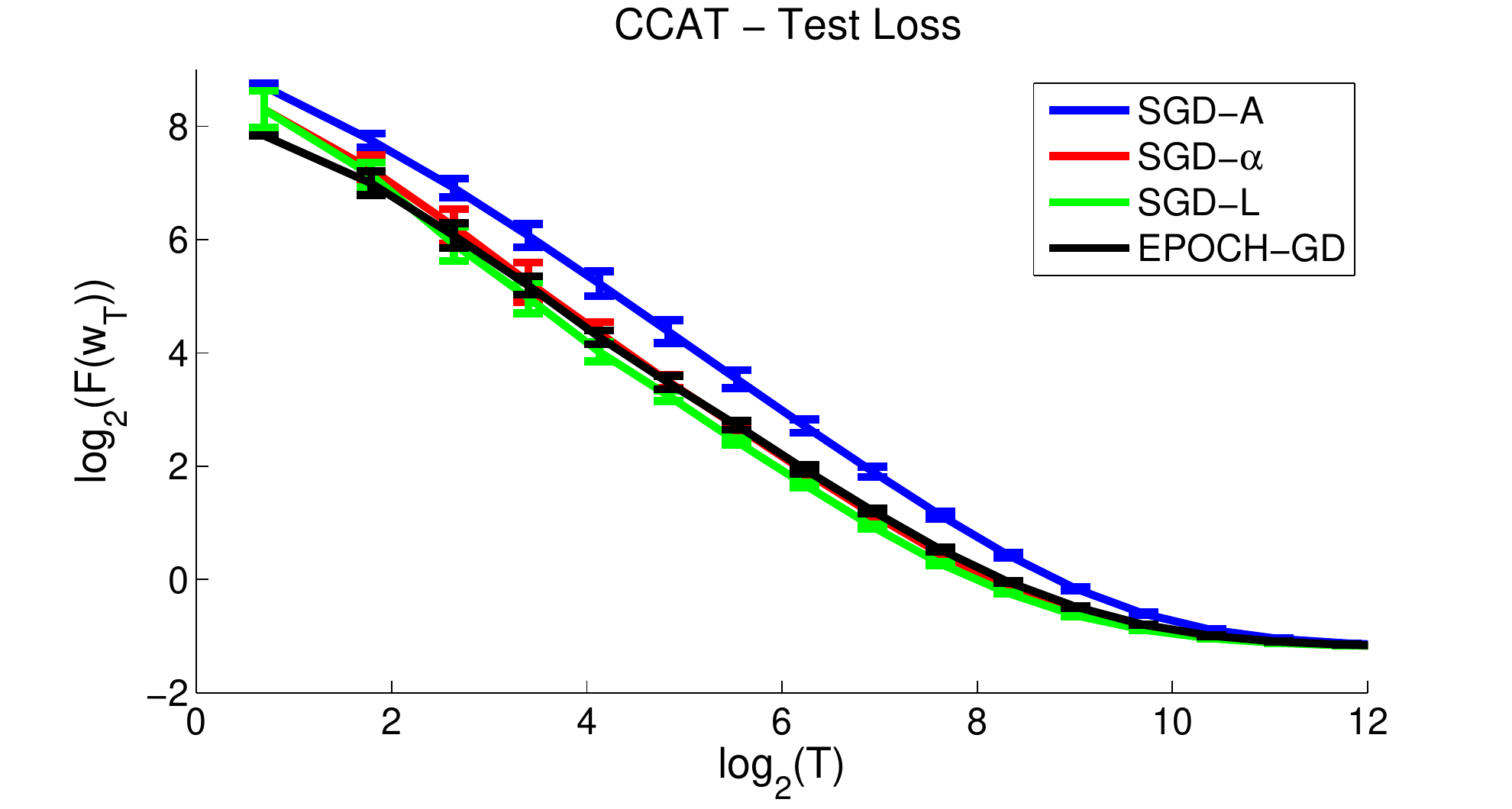}
\caption{Results for the \textsc{ccat} dataset. See \figref{fig:astro} caption for details.}
\label{fig:ccat}
\end{center}
\end{figure}

\begin{figure}
\begin{center}
\includegraphics[scale=0.4]{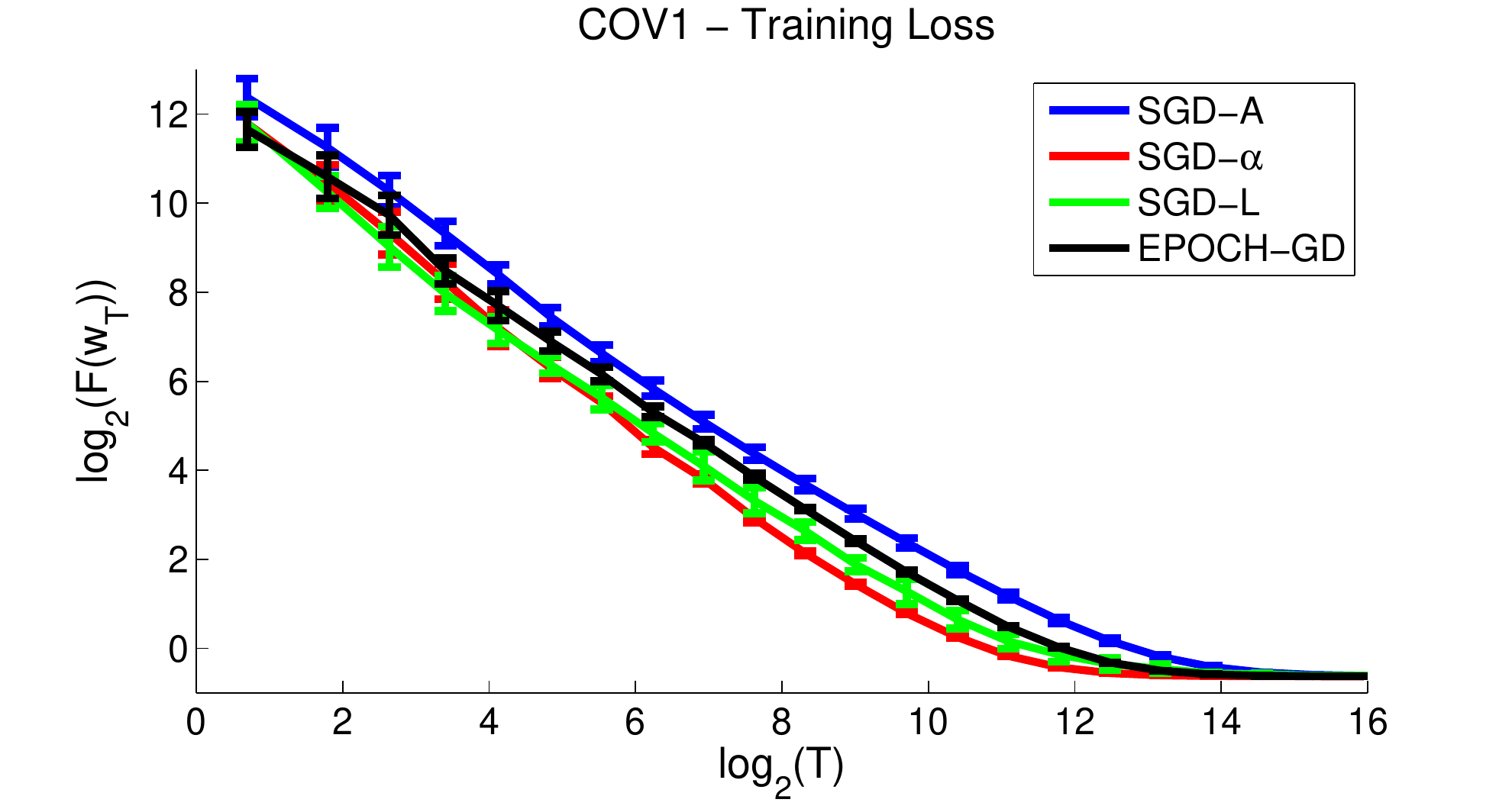}
\includegraphics[scale=0.4]{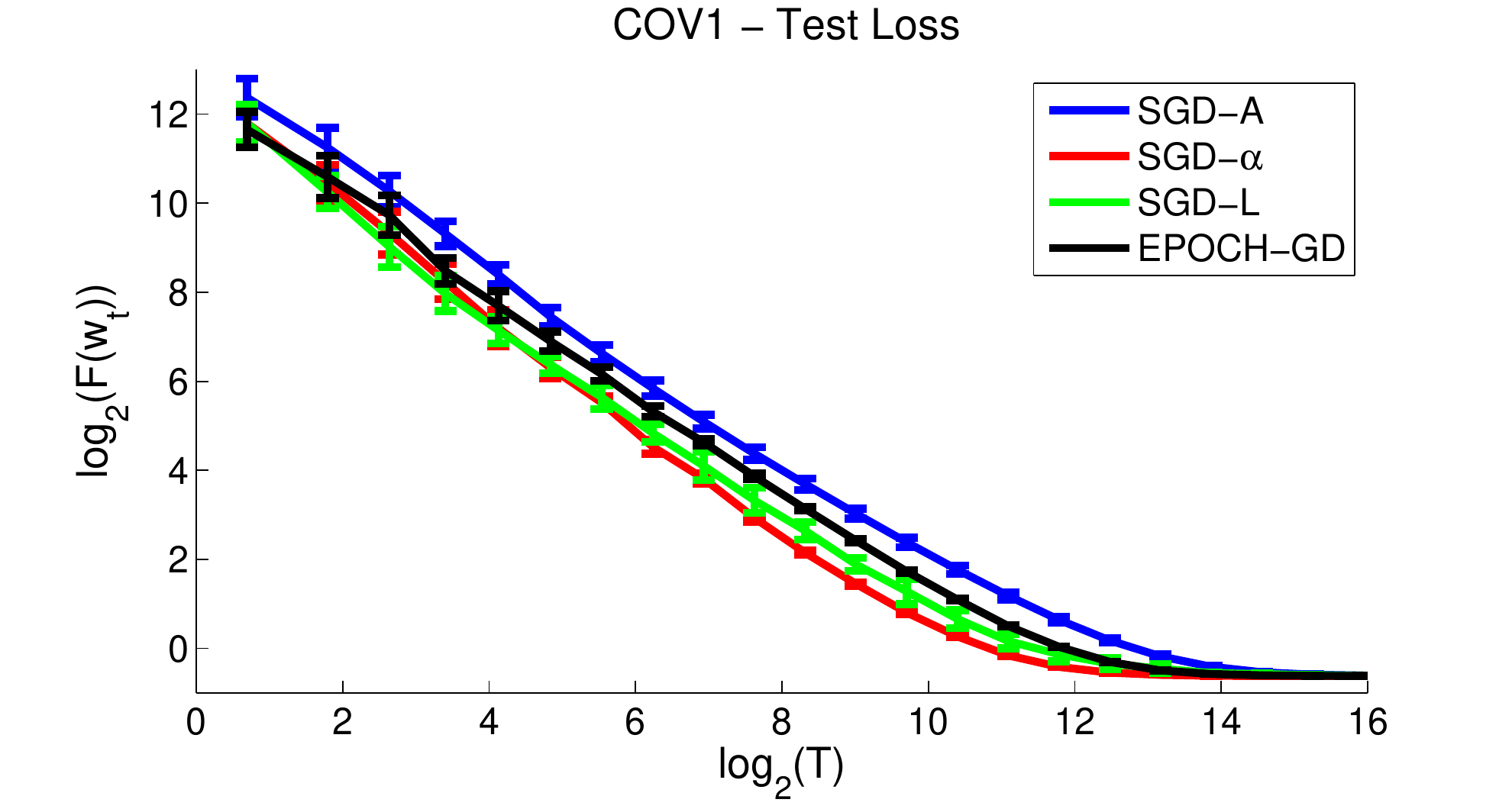}
\caption{Results for the \textsc{ccat} dataset. See \figref{fig:astro} caption for details.}
\label{fig:cov}
\end{center}
\end{figure}

In summary, our experiments indicate the following:
\begin{itemize}
    \item \textsc{Sgd-A}, which averages over all $T$ predictors, is worse than the other approaches. This accords with our theory, as well as the results reported in \cite{ShaSiSreCo11}.
    \item The \textsc{Epoch-Gd} algorithm does have better performance than \textsc{Sgd-A}, but a similar or better performance was obtained using the simpler approaches of $\alpha$-suffix averaging (\textsc{Sgd-$\alpha$}) or even just returning the last predictor (\textsc{Sgd-L}). The good performance of \textsc{Sgd-$\alpha$} is supported by our theoretical results, and so does the performance of \textsc{Sgd-L} in the strongly convex and smooth case.
    \item \textsc{Sgd-L} also performed rather well (with what seems like a $\Theta(1/T)$ rate) on the non-smooth problem reported in \figref{fig:artificial}, although with a larger variance than \textsc{Sgd-$\alpha$}. Our current theory does not cover the convergence of the last predictor in non-smooth problems - see the discussion below.
\end{itemize}

\section{Discussion}\label{sec:discussion}

In this paper, we analyzed the behavior of SGD for strongly convex stochastic optimization problems. We demonstrated that this simple and well-known algorithm performs optimally whenever the underlying function is smooth, but the standard averaging step can make it suboptimal for non-smooth problems. However, a simple modification of the averaging step suffices to recover the optimal rate, and a more sophisticated algorithm is not necessary. Our experiments seem to support this conclusion.

There are several open issues remaining. In particular, the $\Ocal(1/T)$ rate in the non-smooth case still requires some sort of averaging. However, in our experiments and other studies (e.g. \cite{ShaSiSreCo11}), returning the last iterate $\bw_T$ also seems to perform quite well. Our current theory does not cover this - at best, one can use \lemref{lem:close} and Jensen's inequality to argue that the last iterate has a $\Ocal(1/\sqrt{T})$ rate, but the behavior in practice is clearly much better. Does SGD, \emph{without} averaging, obtain an $\Ocal(1/T)$ rate for general strongly convex problems? Also, a fuller empirical study is warranted of whether and which averaging scheme is best in practice.

\textbf{Acknowledgements:} We thank Elad Hazan and Satyen Kale for helpful comments, and to Simon Lacoste-Julien for pointing out a bug in \lemref{lem:close} in a previous version of this paper.

\bibliographystyle{icml2012}
\bibliography{mybib}


\appendix
\newpage
\onecolumn

\section{Justifying $\eta_t=\Theta(1/t)$ Step-Sizes}\label{app:etat}

In this appendix, we justify our focus on the step-size regime $\eta_t=\Theta(1/t)$, by showing that for other step sizes, one cannot hope for an optimal convergence rate in general.

Let us begin by considering the scalar, strongly convex function $F(w)=\frac{1}{2}w^2$, in the deterministic case where $\hat{g}_t = g_t = \nabla F(w_t) = w_t$ with probability $1$, and show that $\eta_t$ cannot be \emph{smaller} than $\Omega(1/t)$. Intuitively, such small step sizes do not allow the iterates $w_t$ to move towards the optimum sufficiently fast. More formally, starting from (say) $w_1=1$ and using the recursive equality $w_{t+1}=w_{t}-\eta_t \hat{g}_t$, we immediately get $w_{T}=\prod_{t=1}^{T-1}(1-\eta_t)$. Thus, if we want to obtain a $\Ocal(1/T)$ convergence rate using the iterates returned by the algorithm, we must at least require that
\[
\prod_{t=1}^{T-1}(1-\eta_t) ~\leq~ \Ocal(1/T).
\]
This is equivalent to requiring that
\[
-\sum_{t=1}^{T-1}\log(1-\eta_t)~\geq~ \Omega(\log(T)).
\]
For large enough $t$ and small enough $\eta_t$, $\log(1-\eta_t)\approx -\eta_t$, and we get that $\sum_{t=1}^{T} \eta_t$ must scale at least logarithmically with $T$. This requires $\eta_t \geq \Omega(1/t)$.

To show that $\eta_t$ cannot be \emph{larger} than $\Ocal(1/t)$, one can consider the function $F(w)=\frac{1}{2}w^2+w$ over the domain $\Wcal=[0,1]$ - this is a one-dimensional special case of the example considered in \thmref{thm:simple}. Intuitively, with an appropriate stochastic gradient model, the random fluctuations in $F(w_{t+1})$ (conditioned on $w_1,\ldots,w_t$) are of order $\eta_t$, so we need $\eta_t=\Ocal(1/t)$ to get optimal rates. In particular, in the proof of \thmref{thm:simple}, we show that for an appropriate stochastic gradient model, $\E[F(w_t)]\geq \E[w_t] \geq \eta_t / 16$ (see \eqref{eq:ewlbound}). A similar lower bound can also be shown for the unconstrained setting considered in \thmref{thm:involved}.

\section{Proofs}\label{app:proofs}

\subsection{Some Technical Results}\label{app:technical}

In this subsection we collect some technical Results we will need for the other proofs.

\begin{lemma}\label{lem:distance}
If $\E[\norm{\hat{\bg}_1}^2]\leq G^2$, then
\[
\E[\norm{\bw_1-\bw^*}^2]\leq  \frac{4G^2}{\lambda^2}.
\]
\end{lemma}

\begin{proof}
Intuitively, the lemma holds because the strong convexity of $F$ implies that the expected value of $\norm{\hat{\bg}_1}^2$ must strictly increase as we get farther from $\bw^{*}$. More precisely, strong convexity implies that for any $\bw_1$,
\[
\inner{\bg_1,\bw_1-\bw^{*}} \geq \frac{\lambda}{2}\norm{\bw_1-\bw^{*}}^2
\]
so by the Cauchy-Schwartz inequality,
\begin{equation}\label{eq:dg1}
\norm{\bg_1}^2\geq \frac{\lambda^2}{4}\norm{\bw_1-\bw^{*}}^2.
\end{equation}
Also, we have that
\[
\E[\norm{\hat{\bg}_1}^2] ~=~ \E[\norm{\bg_1+(\hat{\bg}_1-\bg_1)}^2]
~=~ \E[\norm{\bg_1}^2]+\E[\norm{\hat{\bg}_1-\bg_1}^2]
+2\E[\inner{\hat{\bg}_1-\bg_1,\bg_1}]~\geq~
\E[\norm{\bg_1}^2].
\]
Combining this and \eqref{eq:dg1}, we get that for all $t$,
\[
\E[\norm{\bw_1-\bw^*}^2] \leq \frac{4}{\lambda^2}\E[\norm{\hat{\bg}_t}^2] \leq \frac{4G^2}{\lambda^2}.
\]
\end{proof}

The following version of Freedman's inequality appears in \cite{de1999general} (Theorem 1.2A):
\begin{theorem}\label{thm:Freedman-uniform}
Let $d_1, \dots, d_T$ be a martingale difference sequence with a uniform upper bound $b$ on the steps $d_i$.
Let $V$ denote the sum of conditional variances,
\[
V_s = \sum_{i=1}^s \var(d_i ~|~ d_1, \dots, d_{i-1}).
\]
Then, for every $a,v > 0$,
\[
\Prob{\sum_{i=1}^s d_i \ge a \mbox{ and } V_s \le v ~~\mbox{ for some } s\leq T} \le \exp
\left(\frac{-a^2}{2(v + ba)}\right).
\]
\end{theorem}

The proof of the following lemma is taken almost verbatim from \cite{bartlett2008high}, with the only modification being the use of Theorem~\ref{thm:Freedman-uniform} to avoid an unnecessary union bound.
\begin{lemma}
\label{lem:fromFreedman}
Let $d_1,\ldots,d_T$ be a martingale difference sequence with a uniform bound $|d_i| \leq b$ for all $i$. Let
$V_s = \sum_{t=1}^s \var_{t-1} (d_t)$ be the sum of conditional variances of $d_t$'s. Further, let $\sigma_s = \sqrt{V_s}$.
Then we have, for any $\delta < 1/e$ and $T \ge 4$,
\begin{align}
\Prob{ \sum_{t=1}^s d_t > 2 \max\left\{ 2 \sigma_s, b\sqrt{\log(1/\delta)} \right\} \sqrt{\log(1/\delta)}  ~~\mbox{ for some } s\leq T} \le \log(T) \delta\ .
\end{align}
\end{lemma}
\begin{proof}
Note that a crude upper bound on $\var_t d_t$ is $b^2$. Thus, $\sigma_s \le b\sqrt{T}$. We choose a discretization
$0 = \alpha_{-1} < \alpha_0 < \ldots < \alpha_l$
such that $\alpha_{i+1} = 2\alpha_i$ for $i \ge 0$ and $\alpha_l \ge b\sqrt{T}$. We will specify the choice of
$\alpha_0$ shortly. We then have,
{\allowdisplaybreaks
\begin{align*}
&\Prob{ \sum_{t=1}^s d_t > 2 \max\{ 2 \sigma_s, \alpha_0 \} \sqrt{\log(1/\delta)}  ~~\mbox{ for some } s\leq T } \\
&= \sum_{j=0}^l \Prob{ \begin{matrix}
 \sum_{t=1}^s d_t > 2 \max\{ 2 \sigma_s, \alpha_0 \} \sqrt{ \log(1/\delta)  }  \\
\&\ \alpha_{j-1} < \sigma_s \le \alpha_j
\end{matrix}  ~~\mbox{ for some } s\leq T } \\
&\le \sum_{j=0}^l \Prob{ \begin{matrix}
 \sum_{t=1}^s d_t > 2 \alpha_j \sqrt{ \log(1/\delta) }  \\
\&\ \alpha_{j-1}^2 < V_s \le \alpha_j^2
\end{matrix}  ~~\mbox{ for some } s\leq T} \\
&\le \sum_{j=0}^l \Prob{
\sum_{t=1}^s d_t > 2 \alpha_j \sqrt{ \log(1/\delta) } \ \& \
V_s \le \alpha_j^2   ~~\mbox{ for some } s\leq T}  \\
&\le \sum_{j=0}^l \exp \left(
\frac{-4\alpha_j^2 \log(1/\delta)}
{2\alpha_j^2 + \frac{2}{3}\left( 2 \alpha_j \sqrt{ \log(1/\delta) } \right) b }
\right) \\
&= \sum_{j=0}^l \exp \left(
\frac{-2\alpha_j \log(1/\delta)}
{\alpha_j + \frac{2}{3}\left( \sqrt{ \log(1/\delta) } \right) b }
\right)
\end{align*}
}
where the last inequality follows from Theorem~\ref{thm:Freedman-uniform}.
If we now choose $\alpha_0 = b \sqrt{ \log(1/\delta) }$, then $\alpha_j \ge b \sqrt{ \log(1/\delta) }$
for all $j$. Hence every term in the above summation is bounded by $\exp\left(
\frac{-2\log(1/\delta)}{1 + 2/3}\right) < \delta$. Choosing $l = \log(\sqrt{T})$ ensures that
$\alpha_l \ge b\sqrt{T}$. Thus we have
\begin{align*}
\Prob{  \sum_{t=1}^T X_t > 2 \max\{ 2 \sigma_s, b\sqrt{\log(1/\delta)} \} \sqrt{\log(1/\delta)} }
&= \Prob{ \sum_t X_t > 2 \max\{ 2 \sigma_s, \alpha_0 \} \sqrt{\log(1/\delta)}  } \\
&\le (l+1)\delta = (\log(\sqrt{T})+1)\delta \le \log(T)\delta\ .
\end{align*}
\end{proof}

\subsection{Proof of \lemref{lem:close}}\label{app:close}

By the strong convexity of $F$ and the fact that $\bw^*$ minimizes $F$ in $\Wcal$, we have
\[
\inner{\bg_t,\bw_t-\bw^{*}} \geq F(\bw_t)-F(\bw^{*})+\frac{\lambda}{2}\norm{\bw_t-\bw^{*}}^2,
\]
as well as
\[
F(\bw_t)-F(\bw^*) \geq \frac{\lambda}{2}\norm{\bw_t-\bw^{*}}^2.
\]
Also, by convexity of $\Wcal$, for any point $\bv$ and any $\bw\in \Wcal$ we have $\norm{\Pi_{\Wcal}(\bv)-\bw}\leq \norm{\bv-\bw}$. Using these inequalities, we have the following:
\begin{eqnarray*}
    \E\left[\norm{\bw_{t+1}-\bw^{*}}^2\right] &=& \E[\norm{\Pi_{\Wcal}(\bw_{t}-\eta_t \hat{\bg}_t)-\bw^{*}}^2]\\
    &\leq& \E\left[\norm{\bw_{t}-\eta_t \hat{\bg}_t-\bw^{*}}^2\right]\\
    &=& \E\left[\norm{\bw_t-\bw^{*}}^2\right]-2\eta_t \E[\inner{\hat{\bg}_t,\bw_t-\bw^{*}}]+\eta_t^2\E[\norm{\hat{\bg}_t}^2]\\
    &=& \E\left[\norm{\bw_t-\bw^{*}}^2\right]-2\eta_t \E[\inner{\bg_t,\bw_t-\bw^{*}}]+\eta_t^2\E[\norm{\hat{\bg}_t}^2]\\
    &\leq& \E\left[\norm{\bw_t-\bw^{*}}^2\right]-2\eta_t \E\left[F(\bw_t)-F(\bw^{*})+\frac{\lambda}{2}\norm{\bw_t-\bw^{*}}^2\right]+\eta_t^2 G^2\\
    &\leq& \E\left[\norm{\bw_t-\bw^{*}}^2\right]-2\eta_t \E\left[\frac{\lambda}{2}\norm{\bw_t-\bw^{*}}^2+\frac{\lambda}{2}\norm{\bw_t-\bw^{*}}^2\right]+\eta_t^2 G^2\\
    &=& (1-2\eta_t \lambda)\E\left[\norm{\bw_t-\bw^{*}}^2\right]+\eta_t^2 G^2.
\end{eqnarray*}
Plugging in $\eta_t = 1/\lambda t$, we get
\[
\E\left[\norm{\bw_{t+1}-\bw^{*}}^2\right] \leq \left(1-\frac{2}{t}\right)\E\left[\norm{\bw_t-\bw^{*}}^2\right]+\frac{G^2}{\lambda^2 t^2}.
\]
By \lemref{lem:distance}, we know that $\E\left[\norm{\bw_{t}-\bw^{*}}^2\right]\leq 4G^2/\lambda^2 t$ for $t=1$, and the inequality above implies that $\E\left[\norm{\bw_{t}-\bw^{*}}^2\right]\leq 4G^2/\lambda^2 t$ for $t=2$. The result now follows from a simple induction argument, using $t=3$ as the base case.

\subsection{Proof of \thmref{thm:average}}\label{app:average}
For any $t$, define $\bar{\bw}_t = (\bw_1+\ldots+\bw_t)/t$. Then we have
\begin{eqnarray*}
    \E\left[\norm{\bar{\bw}_{t+1}-\bw^{*}}^2\right] &=& \E\left[\norm{\frac{t}{t+1}\bar{\bw}_t+\frac{1}{t+1}\bw_{t+1}-\bw^{*}}^2\right]\\
    &=& \E\left[\norm{\frac{t}{t+1}(\bar{\bw}_t-\bw^{*})+\frac{1}{t+1}(\bw_{t+1}-\bw^{*})}^2\right]\\
    &=& \left(\frac{t}{t+1}\right)^2\E\left[\norm{\bar{\bw}_t-\bw^{*}}^2\right]+\frac{2t}{(t+1)^2}\E\left[\inner{\bar{\bw}_t-\bw^{*},\bw_{t+1}-\bw^{*}}\right]+\frac{1}{(t+1)^2}\E\left[\norm{\bw_{t+1}-\bw^{*}}^2\right]\\
    &\leq&
    \left(\frac{t}{t+1}\right)^2\E\left[\norm{\bar{\bw}_t-\bw^{*}}^2\right]+\frac{2}{t+1}\E\left[\norm{\bar{\bw}_t-\bw^{*}}~\norm{\bw_{t+1}-\bw^{*}}\right]+\frac{1}{(t+1)^2}\E\left[\norm{\bw_{t+1}-\bw^{*}}^2\right]\\
\end{eqnarray*}
Using the inequality $\E[|XY|]\leq \sqrt{\E[X^2]}\sqrt{\E[Y^2]}$ for any random variables $X,Y$ (which follows from Cauchy-Schwartz), and the bound of \lemref{lem:close}, we get that 
\[
\left(\frac{t}{t+1}\right)^2\E\left[\norm{\bar{\bw}_t
-\bw^{*}}^2\right]+\frac{4G}{ \lambda(t+1)^{3/2}}\sqrt{\E[\norm{\bar{\bw}_t-\bw^*}^2]}+\frac{4G^2}{\lambda^2 (t+1)^3}.
\]
By an induction argument (using \lemref{lem:distance} for the base case), it is easy to verify that
\[
\E\left[\norm{\bar{\bw}_{T}-\bw^*}^2\right] \leq \frac{32G^2}{\lambda^2 T}.
\]
By the assumed smoothness of $F$ with respect to $\bw^*$, we have $F(\bar{\bw}_T)-F(\bw^{*}) \leq \frac{\mu}{2}\norm{\bar{\bw}_T-\bw^{*}}^2$. Combining it with the inequality above, the result follows.

\subsection{Proof of \thmref{thm:simple}}\label{app:simple}

The SGD iterate can be written separately for the first coordinate as
\begin{equation}\label{eq:iterate}
w_{t+1,1} = \Pi_{[0,1]}\left((1-\eta_t)w_{t,1}-\eta_t Z_t\right)
\end{equation}

Fix some $t\geq T_0$, and suppose first that $Z_t\leq -1/2$. Conditioned on this event, we have
\[
w_{t+1,1}~\geq~ \Pi_{[0,1]}\left((1-\eta_t)w_{t,1}+\frac{1}{2}\eta_t\right)
~\geq~\Pi_{[0,1]}(\eta_t/2)~\geq~ \eta_t/2,
\]
since $t\geq T_0$ implies $\eta_t =c/t \leq 2$. On the other hand, if $Z_t>-1/2$, we are still guaranteed that $w_{t+1,1}\geq 0$ by the domain constraints. Using these results, we get
\begin{align}\label{eq:ewlbound}
\E[w_{t+1,1}] ~&=~ \Pr(Z_t\leq -1/2)\E[w_{t+1,1}~|~ Z_t\leq -1/2]+\Pr(Z_t> -1/2)\E[w_{t+1,1}~|~Z_t> -1/2]\notag\\
&\geq~ \Pr(Z_t\leq -1/2)\E[w_{t+1,1}~|~Z_t\leq -1/2]\notag\\
&\geq~ \Pr(Z_t\leq -1/2) \frac{1}{2}\eta_t ~=~ \frac{1}{16}\eta_t.
\end{align}
Therefore,
\[
\E[\bar{w}_{T,1}] ~\geq~ \frac{1}{T}\sum_{t=T_0+1}^{T}\E[w_{t,1}] ~\geq~ \frac{1}{16 T}\sum_{t=T_0+1}^{T}\eta_{t-1}~=~ \frac{1}{16 T}\sum_{t=T_0}^{T-1}\eta_t.
\]
Thus, by definition of $F$,
\[
\E\left[F(\bar{\bw}_{T})-F(\bw^{*})\right] ~=~ \E\left[F(\bar{\bw}_{T})\right] ~\geq~ \E\left[F((\bar{w}_{T,1},0,\ldots,0))\right] ~\geq~
\E\left[\bar{w}_{T,1}\right] ~\geq~ \frac{1}{16T}\sum_{t=T_0}^{T-1}\eta_t.
\]
Substituting $\eta_t = c/t$ gives the required result.

\subsection{Proof of \thmref{thm:involved}}\label{app:involved}

The SGD iterate for the first coordinate is
\begin{equation}\label{eq:iterate2}
w_{t+1,1} = \Pi_{[-1,1]}\left(\left(1-\frac{c}{t}\right)w_{t,1}-\begin{cases}Z_t \frac{c}{t} & w_{t,1}\geq 0\\ -\frac{7c}{t} & w_{t,1} < 0 \end{cases}\right).
\end{equation}
The intuition of the proof is that whenever $w_{t,1}$ becomes negative, then the large gradient of $F$ causes $w_{t+1,1}$ to always be significantly larger than $0$. This means that in some sense, $w_{t,1}$ is ``constrained'' to be larger than $0$, mimicking the actual constraint in the example of \thmref{thm:simple} and forcing the same kind of behavior, with a resulting $\Omega(\log(T)/T)$ rate.

To make this intuition rigorous, we begin with the following lemma, which shows that $w_{t,1}$ can never be significantly smaller than $0$, or ``stay'' below $0$ for more than one iteration.

\begin{lemma}\label{lem:above0}
For any $t\geq T_0=\max\{2,6c+1\}$, it holds that
\[
w_{t,1}\geq -\frac{c}{t-1},
\]
and if $w_{t,1}<0$, then
\[
w_{t+1,1}\geq \frac{5c}{t}.
\]
\end{lemma}

\begin{proof}
Suppose first that $w_{t-1,1}\geq 0$. Then by \eqref{eq:iterate2} and the fact that $Z_t\geq -1$, we get $w_{t,1}\geq -c/(t-1)$. Moreover, in that case,  if $w_{t,1}<0$, then by \eqref{eq:iterate2} and the previous observation,
\[
w_{t+1,1} = \Pi_{[-1,1]}\left(\left(1-\frac{c}{t}\right)w_{t,1}+\frac{7c}{t}\right) \geq \Pi_{[-1,1]}\left(-\left(1-\frac{c}{t}\right)\frac{c}{t-1}+\frac{7c}{t}\right).
\]
Since $t\geq c$, we have $1-c/t \in (0,1)$, which implies that the above is lower bounded by
\[
\Pi_{[-1,1]}\left(-\frac{c}{t-1}+\frac{7c}{t}\right).
\]
Moreover, since $t\geq 6c$, we have $-c/(t-1)+7c/t \in [-1,1]$, so the projection operator is unnecessary, and we get overall that
\[
w_{t+1,1} \geq -\frac{c}{t-1}+\frac{7c}{t} \geq \frac{5c}{t}.
\]
This result was shown to hold assuming that $w_{t-1,1}\geq 0$. If $w_{t-1,1}<0$, then repeating the argument above for $t-1$ instead of $t$, we must have $w_{t,1}\geq 5c/(t-1) \geq -c/(t-1)$, so the statement in the lemma holds also when $w_{t-1,1}<0$.
\end{proof}

We turn to the proof of \thmref{thm:involved} itself. By \lemref{lem:above0}, if $t\geq T_0$, then $w_{t,1}<0$ implies $w_{t+1,1}\geq 0$, and moreover, $w_{t,1}+w_{t+1,1} \geq -c/(t-1)+5c/t\geq 3c/t$. Therefore, we have the following, where the sums below are only over $t$'s which are between $T_0$ and $T$:
\begin{align}
\E[w_{T_0,1}+\ldots+w_{T,1}] ~&=~ \frac{1}{2}~\E[2w_{T_0,1}+\ldots+2w_{T,1}]
~\geq~ \frac{1}{2}\E\left[\sum_{t:w_{t,1}<0}(w_{t,1}+w_{t+1,1})+\sum_{t:w_{t,1}>c/t}w_{t,1} \right]\notag\\
&\geq~ \frac{1}{2}\E\left[\sum_{t:w_{t,1}<0}\frac{3c}{t}+\sum_{t:w_{t,1}> c/t} \frac{c}{t}\right]\notag\\
&\geq \frac{1}{2}\E\left[\sum_{t:w_{t,1}\notin [0,c/t]}\frac{c}{t}\right] = \sum_{t}\Pr(w_{t,1}\notin[0,\eta_t])\frac{c}{2t}.\label{eq:lowb}
\end{align}
Now, we claim that the probabilities above can be lower bounded by a constant. To see this, consider each such probability for $w_{t,1}$, conditioned on the event $w_{t-1,1}\geq 0$. Using the fact that $c/t< 1$, we have
\begin{align*}
&\Pr(w_{t,1}\in \left[0,c/t\right]~|~w_{t-1,1}\geq 0) ~=~\Pr\left(\Pi_{[-1,1]}\left(\left(1-c/t\right)w_{t-1,1}-(c/t)
Z_t\right) \in \left[0,c/t\right]~|~w_{t-1,1}\geq 0\right)\\
&=~\Pr\left(\left(1-c/t\right)w_{t-1,1}-(c/t) Z_t \in \left[0,c/t\right]~|~w_{t-1,1}\geq 0\right)\\
&= \Pr\left(Z_t \in \left[\left(t/c-1\right)w_{t-1,1}-1~,~
\left(t/c-1\right)w_{t-1,1}\right] ~|~ w_{t-1,1}\geq 0\right).
\end{align*}
This probability is at most $1/4$, since it asks for $Z_t$ being constrained in an interval of size $1$, whereas $Z_t$ is uniformly distributed over $[-1,3]$, which is an interval of length $4$. As a result, we get
\[
\Pr(w_{t,1}\notin[0,c/t]~|~w_{t-1,1}\geq 0) ~=~ 1-\Pr(w_{t,1}\in [0,c/t]~|~w_{t-1,1}\geq 0) ~\geq~ \frac{3}{4}.
\]
From this, we can lower bound \eqref{eq:lowb} as follows:
\begin{align*}
&\sum_{t}\Pr(w_{t,1}\notin[0,\eta_t])\frac{c}{t} ~\geq~
\sum_{t}\Pr(w_{t,1}\notin[0,\eta_t],w_{t-1,1}\geq 0)\frac{c}{2t}\\
&=~ \sum_{t}\Pr(w_{t-1,1}\geq 0) \Pr(w_{t,1}\notin[0,\eta_t]~|~ w_{t-1,1}\geq 0)\frac{c}{2t}\\
&\geq~
\frac{3}{8} \sum_{t=T_0}^{T} \frac{c}{t}\Pr(w_{t-1,1}\geq 0)
~=~ \frac{3}{8}\E\left[\sum_{t=T_0}^{T}\frac{c}{t}\mathbf{1}_{w_{t-1,1}\geq 0}\right]\\
&=~ \frac{3}{16}\E\left[\sum_{t=T_0}^{T}\frac{c}{t}\mathbf{1}_{w_{t-1,1}\geq 0}+\sum_{t=T_0}^{T}\frac{c}{t}\mathbf{1}_{w_{t-1,1}\geq 0}\right]
~\geq~ \frac{3}{16}\E\left[\sum_{t=T_0}^{T-1}\frac{c}{t}\mathbf{1}_{w_{t-1,1}\geq 0}+\sum_{t=T_0+1}^{T}\frac{c}{t}\mathbf{1}_{w_{t-1,1}\geq 0}\right]\\
&=~ \frac{3}{16}\E\left[\sum_{t=T_0}^{T-1}\frac{c}{t}\mathbf{1}_{w_{t-1,1}\geq 0}+\sum_{t=T_0}^{T-1}\frac{c}{t+1}\mathbf{1}_{w_{t,1}\geq 0}\right]
~\geq~ \frac{3}{16}\E\left[\sum_{t=T_0+1}^{T-1}\frac{c}{t+1}\left(\mathbf{1}_{w_{t-1,1}\geq 0}+
\mathbf{1}_{w_{t,1}\geq 0}\right)\right].
\end{align*}

Now, by \lemref{lem:above0}, for any realization of $w_{T_0,1},\ldots,w_{T,1}$, the indicators $\mathbf{1}_{w_{t,1}\geq 0}$ cannot equal $0$ consecutively. Therefore, $\mathbf{1}_{w_{t-1,1}\geq 0}+\mathbf{1}_{w_{t,1}\geq 0}$ must always be at least $1$. Plugging it in the equation above, we get the lower bound $(3c/16)\sum_{t=T_0+2}^{T}\frac{1}{t}$. Summing up, we have shown that
\[
\E[w_{T_0,1}+\ldots+w_{T,1}] \geq \frac{3c}{16}\sum_{t=T_0+2}^{T}\frac{1}{t}.
\]
By the boundedness assumption on $\Wcal$, we have\footnote{To analyze the case where $\Wcal$ is unbounded, one can replace this by a coarse bound on how much $w_{t,1}$ can change in the first $T_0$ iterations, since the step sizes are bounded and $T_0$ is essentially a constant anyway.} $w_{t,1}\geq -1$, so
\[
\E[w_{1,1}+\ldots+w_{T_0-1,1}] \geq -T_0.
\]
Overall, we get
\[
\E[\bar{w}_{T,1}] = \frac{1}{T}\E\left[\sum_{t=1}^{T}w_{t,1}\right] \geq
\frac{3c}{16 T}\sum_{t=T_0+2}^{T}\left(\frac{1}{t}\right)-\frac{T_0}{T}.
\]
Therefore,
\[
\E\left[F(\bar{\bw}_{T})-F(\bw^{*})\right] ~=~ \E\left[F(\bar{\bw}_{T})\right] ~\geq~ \E\left[F((\bar{w}_{T,1},0,\ldots,0))\right] \geq
\E\left[\bar{w}_{T,1}\right] ~\geq~ \frac{3c}{16 T}\sum_{t=T_0+2}^{T}\left(\frac{1}{t}\right)-\frac{T_0}{T}.
\]
as required.

\subsection{Proof of \thmref{thm:alphaaveraging}}\label{app:sufproof}
\begin{proof}
Using the derivation as in the proof of \lemref{lem:close}, we can upper bound $\E[\norm{\bw_{t+1}-\bw^{*}}^2]$ by
\[
\E[\norm{\bw_t-\bw^{*}}^2]-2\eta_t \E[\inner{\bg_t,\bw_t-\bw^{*}}]+\eta_t^2G^2.
\]
Extracting the inner product and summing over $t=(1-\alpha)T+1,\ldots,T$, we get
\begin{equation}
\sum_{t=(1-\alpha) T+1}^{T}\E[\inner{\bg_t,\bw_t-\bw^{*}}]
~\leq~ \sum_{t=(1-\alpha) T+1}^{T}\frac{\eta_t G^2}{2} +\notag
\sum_{t=(1-\alpha) T+1}^{T}\left(\frac{\E[\norm{\bw_t-\bw^{*}}^2]}{2\eta_t}-\frac{\E[\norm{\bw_{t+1}-\bw^{*}}^2]}{2\eta_t}\right).
\label{eq:apphalf1}
\end{equation}
By convexity of $F$, $\sum_{t=(1-\alpha) T+1}^{T}\E[\inner{\bg_t,\bw_t-\bw^{*}}]$ is lower bounded by
\[
\sum_{t=(1-\alpha) T+1}^{T}\E[F(\bw_t)-F(\bw^{*})]
~\geq~ \alpha T\E\left[\left(F(\bar{\bw}^{\alpha}_{T})-F(\bw^{*})\right)\right].
\]
Substituting this lower bound into \eqref{eq:apphalf1} and slightly rearranging the right hand side, we get that $\E[F(\bar{\bw}^{\alpha}_{T})-F(\bw^{*})]$ can be upper bounded by
\[
\frac{1}{2\alpha T}\left(\frac{1}{\eta_{(1-\alpha)T+1}}\E[\norm{\bw_{(1-\alpha)T+1}
-\bw^*}^2]
+
\sum_{t=(1-\alpha)T+1}^{T}\E[\norm{\bw_t-\bw^*}^2]
\left(\frac{1}{\eta_t}-\frac{1}{\eta_{t-1}}\right)
+
G^2\sum_{t=(1-\alpha)T+1}^{T}\eta_t\right).
\]
Now, we invoke \lemref{lem:close}, which tells us that with any strongly convex $F$, even non-smooth, we have $\E[\norm{\bw_t-\bw^{*}}^2]\leq \Ocal(1/t)$. More specifically, we can upper bound the expression above by
\[
\frac{2G^2}{\alpha T\lambda^2}\left(\frac{1}{((1-\alpha)T+1)\eta_{(1-\alpha)T+1}}
+
\sum_{t=(1-\alpha)T+1}^{T}\frac{1/\eta_t-1/\eta_{t-1}}{t}\right)
+
\frac{G^2}{2\alpha T}\sum_{t=(1-\alpha) T+1}^{T}\eta_t.
\]
In particular, since we take $\eta_t=1/\lambda t$, we get
\[
\frac{2G^2}{\alpha T \lambda}\left(1+\sum_{t=(1-\alpha)T+1}^{T}\frac{1}{t}\right)
+\frac{G^2}{2\alpha T \lambda}\sum_{t=(1-\alpha)T+1}^{T}\frac{1}{t}.
\]
It can be shown that $\sum_{t=(1-\alpha)T+1}^{T}\frac{1}{t} \leq \log(1/(1-\alpha))$. Plugging it in and slightly simplifying, we get the desired bound.
\end{proof}

\subsection{Proof of \propref{prop:closehighprob}}\label{app:highprob}

To prove this proposition, we will first prove the two auxiliary results, which rewrite $\norm{\bw_{t+1}-\bw^{*}}^2$ in a more explicit form and provide a loose uniform upper bound. We will use the notation $\hat{\bz}_t$ to denote $\bg_t-\hat{\bg}_t$.

\begin{lemma}\label{lem:uniformupperbound}
For all $t$, it holds with probability $1$ that
\[
\norm{\bw_t-\bw^*}\leq \frac{2G}{\lambda}
\]
\end{lemma}
\begin{proof}
The proof is analogous to that of \lemref{lem:distance}, using the stronger condition that $\norm{\bg_t}\leq G$ (which follows from the assumption $\norm{\hat{\bg}_t}\leq G$ with probability $1$). Using strong convexity, we have
\[
G\norm{\bw_t-\bw^*}\geq \norm{\bg_t}\norm{\bw_t-\bw^*}\geq \inner{\bg_t,\bw_t-\bw^*} \geq \frac{\lambda}{2}\norm{\bw_t-\bw^*}^2.
\]
The lemma trivially holds for $\norm{\bw_t-\bw^*}=0$. Otherwise, divide both sides by $\lambda\norm{\bw_t-\bw^*}/2$, and the result follows.
\end{proof}

\begin{lemma} \label{lem:recursive-def}
Under the conditions of \propref{prop:closehighprob}, it holds for any $t\geq 2$ that
\[
\norm{\bw_{t+1}-\bw^{*}}^2 ~\leq~
\frac{2}{\lambda(t-1)t}\sum_{i= 2 }^{t}(i-1)\inner{\hat{\bz}_i,\bw_i-\bw^*}
+\frac{G^2}{\lambda^2 t}.
\]
\end{lemma}

\begin{proof}
By the strong convexity of $F$ and the fact that $\bw^*$ minimizes $F$ in $\Wcal$, we have
\[
\inner{\bg_t,\bw_t-\bw^{*}} \geq F(\bw_t)-F(\bw^{*})+\frac{\lambda}{2}\norm{\bw_t-\bw^{*}}^2,
\]
as well as
\[
F(\bw_t)-F(\bw^*) \geq \frac{\lambda}{2}\norm{\bw_t-\bw^{*}}^2.
\]
Also, by convexity of $\Wcal$, for any point $\bv$ and any $\bw\in \Wcal$ we have $\norm{\Pi_{\Wcal}(\bv)-\bw}\leq \norm{\bv-\bw}$. Using these inequalities, we have the following:
\begin{eqnarray}
    \norm{\bw_{t+1}-\bw^{*}}^2 &=& \norm{\Pi_{\Wcal}(\bw_{t}-\eta_t \hat{\bg}_t)-\bw^{*}}^2\notag\\
    &\leq& \norm{\bw_{t}-\eta_t \hat{\bg}_t-\bw^{*}}^2\notag\\
    &=& \norm{\bw_t-\bw^{*}}^2-2\eta_t \inner{\hat{\bg}_t,\bw_t-\bw^{*}}
    +\eta_t^2\norm{\hat{\bg}_t}^2\notag\\
    &=& \norm{\bw_t-\bw^{*}}^2-2\eta_t \inner{\bg_t,\bw_t-\bw^{*}}+2\eta_t\inner{\hat{\bz}_t,\bw_t-\bw^{*}}
    +\eta_t^2\norm{\hat{\bg}_t}^2\notag\\
    &\leq& \norm{\bw_t-\bw^{*}}^2-2\eta_t (F(\bw_t)-F(\bw^{*}))+\frac{\lambda}{2}\norm{\bw_t-\bw^{*}}^2
    +2\eta_t\inner{\hat{\bz}_t,\bw_t-\bw^{*}}+\eta_t^2 G^2\notag\\
    &\leq& \norm{\bw_t-\bw^{*}}^2-2\eta_t \frac{\lambda}{2}\norm{\bw_t-\bw^{*}}^2+\frac{\lambda}{2}\norm{\bw_t-\bw^{*}}^2
    +2\eta_t\inner{\hat{\bz}_t,\bw_t-\bw^{*}}+\eta_t^2 G^2\notag\\
    &=& (1-2\eta_t \lambda)\norm{\bw_t-\bw^{*}}^2+2\eta_t\inner{\hat{\bz}_t,\bw_t-\bw^{*}}+\eta_t^2 G^2\notag\\
    &=& \left(1-\frac{2}{t}\right)\norm{\bw_t-\bw^{*}}^2+\frac{2}{\lambda t}\inner{\hat{\bz}_t,\bw_t-\bw^{*}}+\left(\frac{G}{\lambda t}\right)^2.\label{eq:development}
\end{eqnarray}
Unwinding this recursive inequality till $t=2$, we get that for any $t\geq 2$,
\[
\norm{\bw_{t+1}-\bw^*}^2 \leq \frac{2}{\lambda}\sum_{i= 2 }^{t}\frac{1}{i}\left(\prod_{j=i+1}^{t}\left(1-\frac{2}{j}\right)\right)
\inner{\hat{\bz}_i,\bw_i-\bw^*}+\frac{G^2}{\lambda^2}\sum_{i= 2 }^{t}\frac{1}{i^2}\prod_{j=i+1}^{t}\left(1-\frac{2}{j}\right).
\]
We now note that
\begin{align*}
\prod_{j=i+1}^{t}\left(1-\frac{2}{j}\right)~=~
\prod_{j=i+1}^{t}\frac{j-2}{j}
~=~
\frac{(i-1)i}{(t-1)t}
\end{align*}
and therefore
\[
\sum_{i= 2 }^{t}\frac{1}{i^2}\prod_{j=i+1}^{t}\left(1-\frac{2}{j}\right)
~=~
\sum_{i= 2 }^{t}\frac{i-1}{i(t-1)t} ~\leq~ \frac{1}{t}
\]
as well as
\[
\sum_{i= 2 }^{t}\frac{1}{i}\prod_{j=i+1}^{t}\left(1-\frac{2}{j}\right)\inner{\hat{\bz}_i,\bw_i-\bw^*}
~=~
\frac{1}{(t-1)t}\sum_{i= 2 }^{t}(i-1)\inner{\hat{\bz}_i,\bw_i-\bw^*}
\]
Plugging this back, we get the desired bound.
\end{proof}

With this result at hand, we are now in a position to prove our high probability bound. Denote $Z_i = \inner{\hat{\bz}_i,\bw_i-\bw^{*}}$. We have that the conditional expectation of $Z_i$, given previous rounds, is $\E_{i-1}[Z_i]  = 0$, and the conditional variance, by Cauchy-Schwartz and the easily verified fact that $\norm{\hat{\bz}_i}=\norm{\hat{\bg}_t-\bg} \leq 2G$, is $\var_{i-1}(Z_i) \leq 4G^2\norm{\bw_i-\bw^*}^2$. By Lemma~\ref{lem:recursive-def}, for all $t\geq 2$,
	\begin{align}
		\label{eq:recurse}
	\norm{\bw_{t+1}-\bw^*}^2 ~\leq~ \frac{2}{\lambda(t-1)t}\sum_{i= 2}^{t}(i-1)Z_i+\frac{G^2}{\lambda^2 t}.
	\end{align}
Considering the sum $\sum_{i=2}^{t}(i-1)Z_i$, we have that the sum of conditional variances satisfies
\[
\sum_{i=2}^t\var_{i-1}((i-1)Z_i) \leq 4G^2 \sum_{i=2}^t (i-1)^2 \norm{\bw_i-\bw^*}^2.
\]
We also have the uniform bound
\begin{equation}\label{eq:domainbound}
|(i-1)Z_i| \leq 2G(t-1) \norm{\bw_i-\bw^*} \leq \frac{4G^2(t-1)}{\lambda},
\end{equation}
where the last upper bound is by \lemref{lem:uniformupperbound}.

We now apply Lemma~\ref{lem:fromFreedman} to the sum of martingale differences $\sum_{i= 2}^{t}(i-1)Z_i$, and get that as long as $T\geq 4$ and $\delta\in (0,1/e)$, then with probability at least $1-\delta$, for all $t\leq T$,
\[
\sum_{i=2}^{t}(i-1)Z_i \leq 8 G \max\left\{\sqrt{\sum_{i=2}^t(i-1)^2 \norm{\bw_i-\bw^*}^2}~,~ \frac{G(t-1)}{\lambda}\sqrt{\log\left(\frac{\log(T)}{\delta}\right)} \right\} \sqrt{\log\left(\frac{\log(T)}{\delta}\right)}.
\]
Plugging this back into \eqref{eq:recurse}, we get that with probability at least $1-\delta$, for all $t\geq 2$,
\begin{align*}
\norm{\bw_{t+1}&-\bw^*}^2 \\
&\leq \frac{16G}{\lambda(t-1)t}\max\left\{\sqrt{\sum_{i=2}^t(i-1)^2 \norm{\bw_{i}-\bw^*}^2}~,~ \frac{G(t-1)}{\lambda}\sqrt{\log\left(\frac{\log(T)}{\delta}\right)} \right\} \sqrt{\log\left(\frac{\log(T)}{\delta}\right)}  +
\frac{G^2}{\lambda^2 t}\\
&\leq \frac{16G\sqrt{\log(\log(T)/\delta)}}{\lambda(t-1)t}\sqrt{\sum_{i=2}^t(i-1)^2 \norm{\bw_{i}-\bw^*}^2}+\frac{16G^2 \log(\log(T)/\delta)}{\lambda^2 t}+\frac{G^2}{\lambda^2 t}\\
&= \frac{16G\sqrt{\log(\log(T)/\delta)}}{\lambda(t-1)t}\sqrt{\sum_{i=2}^t(i-1)^2 \norm{\bw_{i}-\bw^*}^2}+\frac{G^2(16\log(\log(T)/\delta)+1)}{\lambda^2 t}.
\end{align*}
What is left to do now is an induction argument, to show that $\norm{\bw_t-\bw^*}^2 \leq a/t$ for a suitably chosen value of $a$. By \lemref{lem:uniformupperbound}, this certainly holds for $t=1,2$ if $a\geq 8G^2/\lambda^2$. To show the induction step, let us rewrite the displayed inequality above as
\[
\norm{\bw_{t+1}-\bw^*}^2 \leq \frac{b}{(t-1)t}\sqrt{\sum_{i=2}^t(i-1)^2 \norm{\bw_{i}-\bw^*}^2}+\frac{c}{t},
\]
where $b=16G\sqrt{\log(\log(T)/\delta)}/\lambda$ and $c=G^2(16\log(\log(T)/\delta)+1)/\lambda^2$. By the induction hypothesis, $\norm{\bw_{i}-\bw^*}^2\leq a/i$ for all $i=2,\ldots,t$. Therefore, to show that $\norm{\bw_{t+1}-\bw^*}^2\leq a/(t+1)$, it suffices to find a sufficiently large value of $a$ so that
\[
\frac{a}{t+1}\geq \frac{b}{(t-1)t}\sqrt{\sum_{i=2}^t(i-1)^2 \frac{a}{i}}+\frac{c}{t}.
\]
Clearly, it is enough to require
\[
\frac{a}{t+1}\geq \frac{b}{(t-1)t}\sqrt{\sum_{i=2}^t (i-1) a}+\frac{c}{t}
~=~ \frac{b}{(t-1)t}\sqrt{\frac{(t-1)t}{2} a}+\frac{c}{t}
~=~ \frac{b}{\sqrt{2(t-1)t}}\sqrt{a}+\frac{c}{t}.
\]
Multiplying both sides by $t+1$, switching sides, and using the fact that $(t+1)/\sqrt{2(t-1)t}$ as well as $(t+1)/t$ is at most $3/2$ for any $t\geq 2$, it follows that $a$ should satisfy
\[
a-\frac{3b}{2}\sqrt{a}-\frac{3c}{2} \geq 0.
\]
Solving the quadratic inequality, we get that $a$ should satisfy
\[
\sqrt{a} \geq \frac{1}{2}\left(\frac{3b}{2}+\sqrt{\frac{9b^2}{4}+6c}\right).
\]
Taking the square of both sides and using the fact that $(x+y)^2 \leq 2(x^2+y^2)$, it is again sufficient that
\[
a ~\geq~ \frac{1}{2}\left(\frac{9b^2}{4}+\frac{9b^2}{4}+6c\right)
~=~ \frac{9b^2}{4}+3c.
\]
Substituting in the values of $b,c$, we get 
\[
a ~\geq~ \frac{576 G^2\log(\log(T)/\delta)}{\lambda^2}+\frac{3G^2(16\log(\log(T)/\delta)+1)}{\lambda^2}
~=~ \frac{(624\log(\log(T)/\delta)+1)G^2}{\lambda^2}.
\]
Thus, if we take
\[
a ~=~ \frac{(624\log(\log(T)/\delta)+1)G^2}{\lambda^2},
\]
the induction hypothesis holds. Also, recall that for the base case to hold, we also need to assume $a\geq 8G^2/\lambda^2$, but this automatically holds for the value of $a$ above (since we assume $\delta< 1/e$). This gives us the required result.
\end{document}